\documentclass[sts]{imsart}
\usepackage[english]{babel}

\usepackage[utf8]{inputenc} 
\usepackage[T1]{fontenc}    
\usepackage[colorlinks=true,linkcolor=blue]{hyperref}      
\usepackage{url}           
\usepackage{booktabs}       
\usepackage{amsfonts}      
\usepackage{nicefrac}      
\usepackage{microtype}
\usepackage[numbers]{natbib}
\usepackage{centernot}
\usepackage[dvipsnames]{xcolor}
\usepackage{pdflscape}
\usepackage{float}
\usepackage{graphics}
\usepackage{subcaption}
\graphicspath{{./figures/}}
\DeclareGraphicsExtensions{.pdf}
\usepackage{tikz}
\usetikzlibrary{positioning, arrows.meta}

\usepackage{algorithm}
\usepackage{algpseudocode}

\usepackage{mathtools}
\usepackage{amsmath}
\usepackage{tabularx} 
\usepackage{multirow}
\usepackage{caption}            
\usepackage{bm}
\usepackage{amsthm}
\usepackage{amssymb}
\usepackage[capitalize]{cleveref}

\usepackage{comment}

\usepackage{graphicx}

\usepackage{appendix}

\startlocaldefs

\theoremstyle{plain}

\newtheorem{lemma}{Lemma}

\theoremstyle{definition}

\newtheorem{example}{Example}[section]
\newtheorem{remark}{Remark}
\newtheorem{definition}{Definition}[section]

\newcommand{\ci}[3]{#1 \perp\kern-5pt \perp #2 \mid #3}
\endlocaldefs

\begin{document}

\begin{frontmatter}

\title{Null importance: Disentangling relevance for interpretable machine learning}
\runtitle{Null importance}

\begin{aug}
\author[A]{\fnms{Garvesh}~\snm{Raskutti}\ead[label=e1]{raskutti@stat.wisc.edu}}
\author[A]{\fnms{Kris}~\snm{Sankaran}\ead[label=e2]{ksankaran@wisc.edu}}
\author[A]{\fnms{Jiaxin}~\snm{Ye}\ead[label=e3]{jye73@wisc.edu}}

\address[A]{Department of Statistics, University of Wisconsin--Madison\printead[presep={\ }]{e1,e2,e3}.}
\end{aug}

\begin{abstract}
Feature importance is at the heart of many interpretable machine learning methods and frameworks, but the term ``importance'' encompasses several fundamentally different notions of relevance. We develop a unified perspective based on \emph{null importance}: a population-level characterization of when a feature is irrelevant under a specified notion of relevance. We consider standard notions of null importance arising from marginal and conditional statistical relevance, predictive risk, functional invariance, and causal effects, and show how these notions answer different scientific questions. We illustrate the framework through two applications in which the distinction is particularly consequential: algorithmic fairness, where common fairness criteria correspond to different notions of null importance, and genomic perturbation modeling, where different standard choices of relevance lead to different conclusions about what a prediction model has learned. The framework connects three aspects of feature analysis: the scientific question defining relevance, the data and model assumptions that shape how the different null notions relate, and the methods used to assess importance. We establish sufficient conditions under which standard notions of null importance coincide and give counterexamples showing how they diverge when those conditions fail. We then ask which notions of null different families of methods actually target and analyze the conditions under which their zero-importance statistics can identify their targets. Finally, simulations encompassing feature dependence, redundancy, nonlinearity, hidden features and other standard phenomena as well as real-data case studies on attribution in images and multiomics data provide empirical evidence for these theoretical distinctions and their practical consequences for interpreting feature importance methods. Taken together, these results provide a common statistical language for relating scientific questions, data-generating assumptions, and algorithms. This language makes precise the conclusions about feature relevance that a given feature analysis method can legitimately support.
\end{abstract}


\end{frontmatter}

\section{Introduction}

Modern statistical and machine learning methods provide increasingly powerful tools primarily engineered for predictive accuracy. However, applying these models to solve real-world problems successfully demands an understanding of the underlying system. As emphasized by Breiman's distinction between the ``two cultures'' of statistical 
modeling~\cite{breiman2001statistical}, one culture prioritizes predictive accuracy, while another seeks to understand the structure underlying prediction. Truly interpretable machine learning unifies these two cultures: it acknowledges the predictive design of modern models but insists that for real-world deployment, we must 
go beyond asking \emph{can we predict $Y$ from $X$?} to rigorously examine \emph{which features matter, and in what sense?}

This question has no single answer. A feature may be relevant because it is statistically associated with an outcome, provides information beyond other variables, improves predictive performance, changes a fitted prediction function, or has a causal effect. These notions can coincide in simple settings, but they differ systematically in the presence of dependence, redundancy, interactions, confounding, mediation, and model misspecification.

We organize these distinctions around a simple idea: \emph{null importance}. Classical statistics uses a null hypothesis to specify what it means for an effect or relationship to be absent—for example, a zero regression coefficient, independence, or no causal treatment effect~\cite{casella2002statistical,fisher1925statistical}. We use the same logic for feature importance: a \emph{null importance} is a population-level statement specifying what it means for a feature to be irrelevant under a particular notion of relevance.

Null importance provides a common language for connecting three parts of interpretable feature analysis that are often considered separately. First, at the level of the \emph{scientific question}, it specifies what kind of relevance or irrelevance is of interest in a particular application. Second, at the level of the \emph{data and data-generating process}, it allows us to state assumptions under which different notions of relevance coincide, or to understand when they necessarily differ. Third, at the level of the \emph{method}, it provides a way to identify which notion of null importance a feature analysis method targets and to determine whether the assumptions required for that interpretation are satisfied. Our goal is therefore not to identify a universally correct definition of feature importance. Instead, we provide a framework for connecting \emph{scientific questions}, \emph{assumptions about the data and model}, and \emph{feature analysis methods}. 

The framework applies across statistics, machine learning, and scientific applications. Classical feature selection and conditional independence testing formalize statistical notions of irrelevance
\cite{guyon2003introduction,kohavi1997wrappers,Tibshirani1996Lasso,
CandèsFanJansonLv2018Knockoffs,ShahPeters2020GCM}. Predictive importance methods define relevance through changes in predictive risk
\cite{FisherRudinDominici2019,LeiZhuWasserman2018LOCO}. Explainable artificial intelligence uses attribution and sensitivity measures to characterize properties of learned prediction functions
\cite{LundbergLee2017SHAP,SundararajanTalyYan2017IG,Molnar2022Interpretable,
murdoch2019definitions}. Causal inference defines relevance through interventions and counterfactual outcomes
\cite{Pearl2009Causality, peters2017elements}. These literatures developed around different objectives, and the differences among their definitions of relevance are often more consequential than the differences among the algorithms used to estimate them.

The same perspective also clarifies questions outside conventional feature selection. Algorithmic fairness, for example, can be viewed as asking under which notion of relevance a protected attribute should be irrelevant. Demographic parity, equalized odds, individual fairness, proxy fairness, and counterfactual fairness impose different forms of statistical, conditional, functional, or causal nullity
\cite{dwork2012fairness,hardt2016equality,datta2017proxy,
kusner2017counterfactual}. Similarly, in scientific applications such as genomic perturbation modeling, one may ask whether a perturbation is associated with a response, improves prediction, or has a direct causal effect. These represent distinct null questions even when concerning the same variables, a divergence carrying critical consequences for both scientific applications and methods development. Consequently, different methods should neither be swapped interchangeably within an analysis nor benchmarked against one another in methodological studies unless they are explicitly known to target the same underlying notion of null importance.

The remainder of the paper develops this perspective in three stages. Section~2 focuses on null importances at the scientific question level. We define several population-level notions of null importance corresponding to statistical, predictive, functional, and causal relevance. Section~3 asks \emph{what assumptions can we make about the data and data-generating process?} We characterize conditions under which these notions coincide and give counterexamples showing how they separate under dependence, redundancy, interactions, confounding, and model misspecification. Section~4 then asks \emph{which methods should we use?} We organize representative feature selection, attribution, sensitivity, predictive, and causal methods according to the notions of null importance they target and the conditions under which their outputs identify those notions. Section~5 provides a number of synthetic and real data case studies including local image attribution and multiomics that are consistent with the messages in Sections 2-4.

The resulting message is simple: interpretable feature importance is a family of questions whose answers depend on what we mean by relevance, what assumptions we are willing to make about the data, and what method we use to assess it. Null importance provides a common language for making these choices explicit and for connecting a scientific claim to the assumptions and methods needed to support it.

\subsection{Related work and positioning}

Interpretability has been studied from several complementary perspectives. One emphasizes \emph{intrinsic interpretability}: models whose structure is sufficiently transparent that no post-hoc explanation is required \citep{lipton2018mythos,Rudin2019StopExplaining}. A second focuses on \emph{post-hoc explanation} of complex models, including feature attribution, surrogate models, examples, concepts, and counterfactuals. A third emphasizes the \emph{scientific or practical purpose} of an interpretation, asking whether an explanation is relevant to the question and audience of interest. These perspectives have motivated numerous taxonomies of interpretable machine learning and explainable AI \citep{DoshiVelezKim2017RigorousScience,GuidottiMonrealeRuggieriTuriniGiannottiPedreschi2018Survey, murdoch2019definitions,ArrietaEtAl2020ExplainableAI}.

Our work complements these taxonomies by organizing interpretability methods according to a different dimension: the \emph{population-level null notion of relevance} represented by zero importance. This distinction is important because methods that appear similar computationally may answer different questions, while methods from very different methodological families may target the same notion of irrelevance. In particular, the predictive, descriptive, and relevance (PDR) framework of \citep{murdoch2019definitions} emphasizes that relevance depends on the purpose of an interpretation. We make this relevance notion explicit statistically by defining the corresponding population null and studying the assumptions under which it can be identified.

A related literature asks how explanations should be \emph{evaluated}. Doshi-Velez and Kim \citep{DoshiVelezKim2017RigorousScience} called for a rigorous science of interpretability, while subsequent work has proposed criteria including faithfulness, sensitivity, robustness, consistency, and human usefulness. For example, Adebayo et al. \citep{adebayo2018sanity} showed that some saliency methods can be insensitive to the model or labels they purport to explain, motivating systematic ``sanity checks.'' Axiomatic approaches to attribution similarly specify desirable properties that explanation methods should satisfy \citep{SundararajanTalyYan2017IG,LundbergLee2017SHAP}. These perspectives are complementary to ours. Rather than asking only whether an explanation satisfies a generic evaluation criterion, we ask what population-level claim is represented by null importance and whether the proposed statistic actually identifies that claim.

This distinction relates closely to work on \emph{stability, reproducibility, and veridical data science}. \citep{yu2013stability} emphasizes stability under appropriate perturbations as a fundamental requirement for reliable statistical conclusions, while~\citep{yu2020veridical} place stability within the broader Predictability Computability Stability (PCS) framework for trustworthy data science. Stability and null identification address different aspects of reliability: a feature-importance conclusion may be stable but consistently target the wrong notion of relevance, or it may target the appropriate null but be unstable in finite samples. Thus, null importance provides a complementary component of a broader trust standard for feature analysis.

Human-centered perspectives further emphasize that explanations are shaped by the purpose for which they are provided and the needs of their audience \citep{miller2019explanation}. This reinforces our view that there is no single notion of feature relevance. A practitioner concerned with prediction may regard a feature as irrelevant if removing it does not change predictive risk, a scientist may instead require invariance of a response function, and a causal investigator may require that an intervention have no effect. These are different scientific questions rather than competing estimates of a single underlying quantity.

Our perspective is therefore broader than a taxonomy of explainable AI methods and narrower than a general theory of interpretability. We focus on a common statistical question that arises across feature selection, feature attribution, sensitivity analysis, predictive modeling, and causal analysis: \emph{what does it mean for a feature to be irrelevant?} Once a null notion is specified, three questions follow naturally: what assumptions make that null identifiable; when is it equivalent to other notions of irrelevance; and which methods target and identify it? This leads to the central distinction in our paper between a method's \emph{targeted null} and the null that its zero-importance statistic actually \emph{identifies}. The subsequent sections develop these connections formally and examine their consequences through theoretical results, counterexamples, simulations, and applications.

\section{Null Importance: Definitions and Examples}
\label{sec:null_notions}

The notion of feature importance depends fundamentally on the underlying notion of scientific relevance. To make these notions precise, let $X = (X_1,\ldots,X_p) \in \mathcal{X}$ denote a vector of features, where $\mathcal{X} = \mathcal{X}_1 \times \cdots \times \mathcal{X}_p$, and let $Y \in \mathcal{Y}$ denote a response variable.

\begin{definition}[Null Importance]
Feature $X_j$ satisfies the null importance property, denoted by $\mathcal{N}_j$, if $X_j$ is deemed uninformative for $Y$ under a chosen criterion of relevance.
\end{definition}

In many settings, relevance is quantified via a non-negative function $R_j \ge 0$ defined on the space of feature-response relationships, where larger values correspond to greater relevance. In this case, the null importance property is characterized by:
\[
\mathcal{N}_j \iff R_j = 0.
\]

The remainder of this section introduces several commonly used notions of null importance arising from different criteria of scientific relevance.

\subsection{Functional Null Importance}

Functional relevance is determined entirely by a prediction function $f: \mathcal{X} \rightarrow \mathcal{Y}$ and does not depend on a data distribution. This notion underlies gradient-based attribution methods such as saliency maps~\cite{SimonyanVedaldiZisserman2014Saliency}, Integrated Gradients~\cite{SundararajanTalyYan2017IG}, and related interpretability approaches~\cite{BachEtAl2015LRP, RibeiroSinghGuestrin2016LIME, SelvarajuCogswellEtAl2017GradCAM, ShrikumarGreensideKundaje2017DeepLIFT, SmilkovEtAl2017SmoothGrad}.

Let $\mathcal{X} = \mathcal{X}_1 \times \cdots \times \mathcal{X}_p$, and let $\mathcal{X}_{-j} = \prod_{k \neq j} \mathcal{X}_k$ denote the subspace excluding the $j$-th feature, where $x_{-j} \in \mathcal{X}_{-j}$ represents a vector $x \in \mathcal{X}$ omitting its $j$-th coordinate $x_j$. For simplicity of notation we restrict to the setting of prediction and binary classification. Natural extensions to multi-class classification, generation and other settings exist.

\begin{definition}[Global Functional Null]
\label{defn:global_functional}
A feature $X_j$ satisfies the \emph{global functional null property}, denoted by $\mathcal{N}_j^{\text{func-global}}$, if the prediction function $f$ does not depend on $x_j$ anywhere on $\mathcal{X}$. That is, there exists a function $g: \mathcal{X}_{-j} \rightarrow \mathcal{Y}$ such that
\[
f(x) = g(x_{-j}) \quad \text{for all } x \in \mathcal{X}.
\]
In terms of our general framework, this corresponds to
$R_j^{\text{global}}(f) = 0$, where
\[
R_j^{\text{global}}(f)
\equiv
\sup_{\substack{x,x' \in \mathcal{X}\\ x_{-j}=x'_{-j}}}
|f(x)-f(x')|.
\]
\end{definition}

\begin{definition}[Local Functional Null]
\label{defn:local_functional}
Let $x \in \mathcal{X}$. A feature $X_j$ satisfies the \emph{local functional null property at $x$}, denoted by $\mathcal{N}_j^{\text{func-local}}(x)$, if there exists an open neighborhood $U(x) \subseteq \mathcal{X}$ and a function $g: \mathcal{X}_{-j} \rightarrow \mathcal{Y}$ such that
\[
f(z) = g(z_{-j}) \quad \text{for all } z \in U(x).
\]
Equivalently, the prediction function is locally invariant to perturbations of $x_j$ near $x$, corresponding to $R_j^{\text{local}}(f; x) = 0$ over $U(x)$.
\end{definition}

\begin{definition}[Pointwise Functional Null]
Suppose $\mathcal{X} \subseteq \mathbb{R}^p$ and $f$ is differentiable at $x \in \mathcal{X}$. A feature $X_j$ satisfies the \emph{pointwise functional null property at $x$}, denoted by $\mathcal{N}_j^{\text{func-pointwise}}(x)$, if
\[
\frac{\partial f(x)}{\partial x_j} = 0.
\]
Here, the local relevance measure reduces to $$R_j^{\text{pointwise}}(f; x) = \left| \frac{\partial f(x)}{\partial x_j} \right| = 0.$$
\end{definition}

\begin{remark}
The three notions above capture functional relevance across different spatial scales. Under appropriate smoothness conditions on $\mathcal{X}$ and $f$, these properties exhibit a natural hierarchy for all $x \in \mathcal{X}$:
\[
\mathcal{N}_j^{\text{func-global}} \implies \mathcal{N}_j^{\text{func-local}}(x) \implies \mathcal{N}_j^{\text{func-pointwise}}(x).
\]
\end{remark}
\subsection{Statistical Null Importance}

In probabilistic settings where $(X, Y)$ is governed by a joint distribution $P(X, Y)$, statistical null importance formalizes feature relevance using probability theory and information theory. We focus on two primary model-agnostic notions: \emph{marginal} and \emph{conditional} statistical null importance.

\begin{definition}[Statistical Marginal Null Importance]
Feature $X_j$ satisfies the \emph{marginal statistical null property}, denoted by $\mathcal{N}_j^{\mathrm{stat-marg}}$, if its marginal mutual information with the response vanishes:
\[
R_j^{\mathrm{marg}} \equiv I(X_j; Y) = 0.
\]
Equivalently, $X_j$ is statistically independent of $Y$, denoted by $Y \perp X_j$.
\end{definition}

Conversely, $X_j$ is marginally relevant whenever we have $R_j^{\mathrm{marg}} > 0$. Methods such as correlation screening~\cite{FanLv2008SIS, FanSong2010SISGLM}, mutual information screening~\cite{PengLongDing2005mRMR}, and nonparametric independence measures~\cite{Gretton2005KernelIndependence, Gretton2008HSIC, SzékelyRizzoBakirov2007DistanceCovariance} directly evaluate or estimate this marginal quantity.

\begin{definition}[Statistical Conditional Null Importance]
\label{defn:cond_stat}
Feature $X_j$ satisfies the \emph{conditional statistical null property}, denoted by $\mathcal{N}_j^{\mathrm{stat-cond}}$, if its conditional mutual information with the response given all other features vanishes:
\[
R_j^{\mathrm{cond}} \equiv I(X_j; Y \mid X_{-j}) = 0.
\]
Equivalently, $X_j$ is conditionally independent of $Y$ given $X_{-j}$, denoted by $Y \perp X_j \mid X_{-j}$.
\end{definition}

Thus, $X_j$ is conditionally null if it carries no incremental information about $Y$ beyond what is already contained in $X_{-j}$. Conversely, $X_j$ is conditionally relevant whenever $R_j^{\mathrm{cond}} > 0$. 

Conditional statistical relevance is among the most widely adopted notions of feature importance in statistical learning~\cite{akaike1974new, Tibshirani1996Lasso, kohavi1997wrappers, guyon2003introduction}. It serves as the foundational property behind conditional independence testing~\cite{ShahPeters2020GCM, BerrettSamworthYuan2021CIT} and distribution-free feature selection frameworks such as Model-X Knockoffs~\cite{BarberCandes2015Knockoffs, CandèsFanJansonLv2018Knockoffs}.

\subsection{Risk-Based Null Importance}

Risk-based relevance quantifies the intrinsic predictive value of a feature with respect to a designated hypothesis class and loss function. This perspective is closely related to predictive feature importance measures such as LOCO~\cite{LeiZhuWasserman2018LOCO}, permutation importance~\cite{Breiman2001RandomForests, FisherRudinDominici2019}, and predictive model reliance~\cite{FisherRudinDominici2019}.

Let $L: \mathcal{Y} \times \mathcal{Y} \rightarrow \mathbb{R}_{\ge 0}$ denote a non-negative loss function, and let 
\[
\mathcal{R}(f) \equiv \mathbb{E}_{(X,Y) \sim P}\left[ L(Y, f(X)) \right]
\]
denote the population risk of a prediction function $f: \mathcal{X} \rightarrow \mathcal{Y}$.

Let $\mathcal{F} \subseteq \{f: \mathcal{X} \rightarrow \mathcal{Y}\}$ denote a class of prediction functions over the full feature space $\mathcal{X}$, and let $\mathcal{F}_{-j} \subseteq \{g: \mathcal{X}_{-j} \rightarrow \mathcal{Y}\}$ denote the restricted class operating on $\mathcal{X}_{-j}$. We define the optimal achievable risks as
\[
R^* \equiv \inf_{f \in \mathcal{F}} \mathcal{R}(f) \quad \text{and} \quad R^*_{-j} \equiv \inf_{g \in \mathcal{F}_{-j}} \mathbb{E}\left[ L(Y, g(X_{-j})) \right].
\]
Assuming $\mathcal{F}_{-j}$ is embedded in $\mathcal{F}$ via the inclusion $f(x) = g(x_{-j})$, it holds that $R^*_{-j} \ge R^*$.

\begin{definition}[Risk-Based Relevance Measure]
\label{defn:risknull}
The risk-based relevance of feature $X_j$ relative to the tuple $(P, \mathcal{F}, L)$ is defined as
\[
R_j^{\mathrm{risk}}(\mathcal{F}, L) \equiv R^*_{-j} - R^*.
\]
\end{definition}

\begin{definition}[Risk-Based Null Importance]
Feature $X_j$ satisfies the \emph{risk-based null property}, denoted by $\mathcal{N}_j^{\mathrm{risk}}$, if its risk-based relevance vanishes:
\[
\mathcal{N}_j^{\mathrm{risk}} \iff R_j^{\mathrm{risk}}(\mathcal{F}, L) = 0 \iff R^*_{-j} = R^*.
\]
\end{definition}

Conversely, $X_j$ is risk-relevant whenever $R_j^{\mathrm{risk}}(\mathcal{F}, L) > 0$, meaning that excluding $X_j$ strictly deteriorates the population risk achievable by the model family $\mathcal{F}$.

This notion forms the basis of many predictive feature importance and feature selection procedures, including LOCO and related predictive relevance methods.

\subsection{Causal Null Importance}

Causal relevance evaluates whether an interventional perturbation of $X_j$ alters the distribution of $Y$. This notion is rooted in the causal inference literature and intervention-based reasoning~\cite{Pearl2009Causality, peters2017elements, HeskesEtAl2020CausalShapley, JanzingMinoricsBlobaum2020CausalFeatureRelevance, KocaogluEtAl2018CausalGAN}.

Let $P(Y \mid \operatorname{do}(X_j = x_j))$ denote the interventional response distribution under Pearl's $do$-calculus, and let $\mathcal{P}_{\mathcal{Y}}$ denote the space of probability measures on $\mathcal{Y}$.
Let $D: \mathcal{P}_{\mathcal{Y}} \times \mathcal{P}_{\mathcal{Y}} \rightarrow \mathbb{R}_{\ge 0}$ be a statistical distance or divergence on probability measures (e.g., Total Variation distance, Kullback-Leibler divergence, or Wasserstein distance) satisfying $D(P_1, P_2) = 0 \iff P_1 = P_2$.

\begin{definition}[Causal Relevance Measure]
Causal relevance of feature $X_j$ may be defined as the supremum interventional divergence over all pairs of feature values $x_j, x_j' \in \mathcal{X}_j$,
\begin{align*}
R_j^{\mathrm{causal}} \equiv \sup_{x_j, x_j' \in \mathcal{X}_j} D\big(P(Y \mid \operatorname{do}(X_j = x_j)), \\ P(Y \mid \operatorname{do}(X_j = x_j')) \big).
\end{align*}
\end{definition}

\begin{definition}[Causal Null Importance]
Feature $X_j$ satisfies the \emph{causal null property}, denoted by $\mathcal{N}_j^{\mathrm{causal}}$, if its causal relevance vanishes:
\[
\mathcal{N}_j^{\mathrm{causal}} \iff R_j^{\mathrm{causal}} = 0.
\]
Equivalently, $P(Y \mid \operatorname{do}(X_j = x_j)) = P(Y \mid \operatorname{do}(X_j = x_j'))$ for all $x_j, x_j' \in \mathcal{X}_j$, rendering the interventional response distribution strictly invariant to $x_j$. Conversely, $X_j$ is causally relevant whenever $R_j^{\mathrm{causal}} > 0$.
\end{definition}

This notion captures scientific questions concerning causal influence and intervention effects rather than statistical association or predictive utility.

\subsection{Example: Algorithmic Fairness as Null Importance}

The null importance framework provides a unifying mathematical language for
algorithmic fairness. Rather than viewing different fairness criteria as
fundamentally different mathematical objects, we interpret each criterion as
specifying a notion of relevance under which the protected attribute should be
irrelevant. This perspective complements existing taxonomies of algorithmic
fairness
\cite{caton2024fairness, dwork2012fairness,VermaRubin2018FairnessDefinitions}
by emphasizing the underlying notion of null importance.

Let $A$ denote a protected attribute, let $X=(X_1,\ldots,X_p)$ denote the
remaining observed features, let $Y$ denote the true outcome, and let
$\hat Y=f(X,A)$ denote the prediction of a learned model.
In this setting, the relevance of the attribute $A$ is  characterized by a
non-negative relevance functional
\[
R_A \ge 0,
\]
with the corresponding null importance property $\mathcal N_A$.

Algorithmic fairness is therefore naturally formulated as requiring that the
protected attribute satisfy an appropriate null importance property. Different
fairness definitions correspond to different choices of relevance functional.

\paragraph{Marginal Statistical Null Importance}

Demographic parity
\cite{dwork2012fairness,barocas2016big}
requires that the prediction be statistically independent of the protected
attribute over the population.

Define the marginal statistical relevance
\[
R_A^{\mathrm{marg}}(\hat Y)
=
I(A;\hat Y).
\]

The corresponding fairness criterion 
$R_A^{\mathrm{marg}}(\hat Y)=0$
is therefore exactly 
$A\perp \hat Y$.
Thus demographic parity is precisely the marginal statistical null importance
property applied to the protected attribute.

\paragraph{Conditional Statistical Null Importance}

Equalized odds
\cite{hardt2016equality}
requires predictions to be conditionally independent of the protected
attribute given the true outcome.

Define
\[
R_A^{\mathrm{cond}}(\hat Y\mid Y)
=
I(A;\hat Y\mid Y).
\]

The corresponding fairness criterion $R_A^{\mathrm{cond}}(\hat Y\mid Y)=0$ is $A\perp\hat Y\mid Y$. Similarly, equality of opportunity
\cite{hardt2016equality}
corresponds to
$R_A^{\mathrm{cond}}(\hat Y\mid Y=1)=0$
and hence only requires conditional independence among positively labeled examples.

\paragraph{Proxy Fairness}

Rather than asking whether predictions depend on the protected attribute,
proxy fairness asks whether an observed feature uniquely encodes protected
information
\cite{datta2017proxy,kilbertus2017avoiding,veale2018fairness}.

For feature $X_j$, define
\[
R_j^{\mathrm{cond}}(A)
=
I(X_j;A\mid X_{-j}).
\]

The corresponding null importance property is $R_j^{\mathrm{cond}}(A) \allowbreak = 0.$
Thus proxy detection is itself a conditional statistical null importance
problem, with the protected attribute serving as the response variable.

\paragraph{Local Functional Null Importance}

Individual fairness
\cite{dwork2012fairness}
is a local property of the prediction function rather than a population-level
property. Informally, similar individuals should receive similar predictions.

Suppose $f$ is differentiable. A natural local functional relevance measure is
\[
R_A^{\mathrm{local}}(x)
=
\left|
\frac{\partial f(x)}{\partial A}
\right|.
\]

The corresponding null importance property is $R_A^{\mathrm{local}}(x) \allowbreak = 0$.
More generally, local functional relevance may be characterized through local
invariance of the prediction function, as given in Definition \ref{defn:local_functional}. Consequently, individual fairness is naturally
interpreted as a local functional null importance criterion.

\paragraph{Global Causal Null Importance}

Counterfactual fairness
\cite{kusner2017counterfactual}
is defined through interventions on the protected attribute.

Define
\begin{align*}
R_A^{\mathrm{causal}}
=
\sup_{a,a'}
D\!\Big(
&P(\hat Y\mid \mathrm{do}(A=a)),\\
&P(\hat Y\mid \mathrm{do}(A=a'))
\Big),
\end{align*}

where $D(\cdot,\cdot)$ is any divergence satisfying
$D(P,Q)=0$ if and only if $P=Q$.

The corresponding null importance property is $R_A^{\mathrm{causal}}=0$, or equivalently,
\[
P(\hat Y\mid \mathrm{do}(A=a))
=
P(\hat Y\mid \mathrm{do}(A=a'))
\qquad
\forall\,a,a'.
\]

Thus counterfactual fairness is precisely a global causal null importance
criterion.

The preceding examples illustrate that many widely-used notions of algorithmic
fairness differ primarily in the notion of relevance they seek to eliminate
rather than in their mathematical structure
\cite{mehrabi2021survey,caton2024fairness,VermaRubin2018FairnessDefinitions}.
Within the null importance framework,

\begin{center}
\small
\begin{tabular}{@{}p{0.38\linewidth}p{0.5\linewidth}@{}}
\toprule
Fairness notion & Null importance interpretation\\
\midrule
Demographic parity & Marginal statistical\\
Equalized odds & Conditional statistical\\
Equality of opportunity & Conditional statistical (restricted outcome)\\
Proxy fairness & Conditional statistical (target = protected attribute)\\
Individual fairness & Local functional\\
Counterfactual fairness & Causal\\
\bottomrule
\end{tabular}
\end{center}
From this perspective, algorithmic fairness is fundamentally the problem of
determining under which notion of relevance the protected attribute should
become irrelevant. This is consistent with prior impossibility results ~\cite{chouldechova2017fair, kleinberg2016inherent} asserting that fairness measures can not simultaneously satisfy multiple fairness criteria. The choice of statistical, functional, predictive, or causal
relevance and whether relevance is assessed globally or locally determines the
appropriate mathematical formulation and corresponding fairness methodology.

\subsection{Example: Clarifying Perturbation Modeling via Null Importance}

Fairness is an unusually clean case. In most scientific problems, the right null is harder to find, and finding it is the scientific work itself. We illustrate this with genomic perturbation modeling, where recent debates about how to evaluate the models are, in large part, disagreements about which null to target \cite{vollenweider2026signalm, nicol2026spurious, liu2025effects}. These models are trained from perturbation sequencing data, where individual genes are knocked down and profiles of the transcriptome are recorded in response. When coupled with sufficiently powerful prediction models, such experiments can uncover the regulatory architecture of the genome. In this context, null importance clarifies the object that is being declared relevant, like the influence of a perturbation target $G$ on the entire response $Y$ or an individual perturbation-response edge $g \to j$, suggesting relevance objects $R_{G}$ and $R_{g \to j}$.
Further, by distinguishing between null notions, the framework resolves disagreements around metric choice, like whether an evaluation should be restricted to differentially expressed genes.

Let $\mathcal{G}$ represent the collection of perturbation targets, $\varnothing$ the unperturbed control, and $\mathcal G^{+}=\mathcal G\cup\{\varnothing\}$. Let $G \sim Q$ be a random perturbation from $\mathcal{G}^{+}$ and $A = \mathbf{1}\{G \neq \varnothing\}$ indicate that a cell has been perturbed. Suppose the transcriptome-wide response across $d$ genes is drawn according to $Y \sim P\left(\cdot \mid \mathrm{do}\left(G\right)\right)$. Let $f$ denote a perturbation prediction model, whose goal is to predict the response $Y$ given an observed perturbation $G$. First we consider notions of relevance useful in characterizing good metrics. 

\textbf{Metric Confounding}. The mere act of perturbing cells can affect the transcriptome, regardless of the choice of perturbed gene. This creates an artificial signal that can make a model appear effective even though it hasn't learned any perturbation-specific response; i.e., it learns a pattern shared across all perturbations that differentiates them from the unperturbed controls, but not the response to any particular perturbation. A model confounded in this way exhibits,
\begin{align*}
    I\left(A; f\left(G\right)\right) > 0, \text{ while } I\left(G; f\left(G\right) \vert A = 1\right) = 0
\end{align*}
where $I(.,.)$ refers to mutual information meaning the model has learned a perturbation effect but a gene-specific signal. From the null importance perspective, the critique is that existing models are evaluating the relevance of the existence of a perturbation $A$ rather than the specific perturbed gene $G$.

\textbf{Metric Baselines}. It is possible to set a baseline representing signal that can be captured by $f$ that is unrelated to the perturbation. Define the risk,
$$
\mathcal R(f)
=
\mathbb E_{G\sim Q}\,
\mathbb E_{Y\sim P(\cdot\mid \mathrm{do}(G))}
\left[L\left(Y,f(G)\right)\right]
$$
for loss function $L(.,.)$. Then, letting $\mathcal{F}_0$ represent all constant functions, the ``strong baseline'' recommendation of \cite{vollenweider2026signalm} can be written as computing
$$
R_G^{\mathrm{risk}}
= \inf_{f'\in\mathcal F_0}\mathcal R(f') - \mathcal{R}\left(f\right),
$$
and using $R_{G}^{\text{risk}}$ to indicate that the model has no per\-tur\-ba\-tion-specific performance improvements. These recommendations can be read as a model-level notion of relevance used to evaluate overall prediction performance, which is distinct from the relevance of specific perturbation-to-target edges $g \to j$. Next we consider how null importance can highlight these biological relationships. 

\textbf{Marginal Importance and DE filtering}. Methods like SCEPTRE \cite{barry2021sceptre} identify marginal associations between response gene $j$ and the fixed knockdown target $g$,
$$
R_{g\to j}^{\mathrm{marg}}
=
D\!\left(
P(Y_j\mid \mathrm{do}(g)),\,
P(Y_j\mid\varnothing)
\right)
$$
where $D(.,.)$ refers to KL-divergence. Typically the perturbation interventions are assumed to satisfy the conditions of Lemma \ref{LemMargCausal} below, allowing for this causal quantity to be evaluated from marginal statistics.

Some authors have advocated for restricting evaluation to a subset of genes that pass a differentially expression (DE) filter \cite{mejia2025}. This corresponds to ignoring edges in a null set defined as edges where the relevance is too low or the target gene is not sufficiently differentially expressed,
$$
\mathcal N^{\varepsilon}_{g\to j}
\iff
R_{g\to j}^{\mathrm{marg}}<\varepsilon,
$$
where $\varepsilon$ is a threshold associated with the power of the underlying differential test.  As the sample size grows, power increases, allowing us to reduce the DE threshold $\varepsilon \downarrow 0$ and returning us to the definition of null importance in the unfiltered case. DE filtering changes not only the set of genes under consideration but also the implicit null notion. It requires evaluation of edges that have already been judged nonnull under a marginal perturbation-response notion.

\textbf{Causal Importance and Drug Targeting}. Let $M$ denote the response of genes other than $j$, i.e., $M = Y_{-j}$ and define,
\begin{align*}
R^{\mathrm{causal}}_{g\to j}
= \sup_{m}\, D\Big( &P(Y_j\mid \mathrm{do}(g), M{=}m), \\
&P(Y_j\mid\varnothing, M{=}m) \Big),
\end{align*}as the direct effect of perturbing gene $g$ on gene $j$ and ignoring effects mediated through all other genes. In this case an edge is null if and only if 
$$
P(Y_j\mid \mathrm{do}(g),M{=}m)=P(Y_j\mid\varnothing,M{=}m) \text{ for all }m
$$
where conditioning on all mediator genes removes the possibility of backdoor confounding~\cite{Pearl2009Causality}. This perspective distinguishes core vs. peripheral genes in the omnigenic model \cite{Ota2025} and is arguably more scientifically aligned with the goal of drug targeting. Indeed a downstream gene $j$ may be correlated with $g$, and this relationship may be learned by the black box $f$ without $g$ being an effective target for modifying $j$'s activity. A perturbation $g$ may be nonnull under $R_{g\to j}^{\text{marg}}$ if changes the response $j$ downstream, but it still may fail to be an effective intervention target in the sense $R_{g \to j}^{\text{causal}} = 0$ if it doesn't impact the response $j$ directly.

These examples highlight that both metrics and methods in perturbation modeling can disagree because they consider different objects to call relevant (e.g., $A$ vs. $G$) or because they target different notions of null importance ($R^{\text{marginal}}_{g \to j}$ vs. $R^{\text{causal}}_{g \to j}$). The language of null importance offers a simple way of organizing discussion around these subtle scientific issues.

\section{Equivalences and counter-examples for null importance}
\label{sec:equivalences_and_counterexamples}

Once a scientific notion of relevance has been specified, the next question is what can be inferred from the available data. Different notions of null importance are generally distinct, but they can coincide under particular assumptions about the data-generating process, the dependence structure among features, the predictive function class, or the relationship between observational and interventional distributions.

This section therefore asks a different question from the previous one. Section~2 specified \emph{what we mean by relevance}; here we ask \emph{when different meanings of relevance lead to the same conclusion}. These equivalences are useful because they explain when a method developed for one notion of relevance can legitimately be interpreted in terms of another. The counterexamples are equally important: they identify common settings in which such interpretations fail. Furthermore, practitioners may be able to determine whether the assumptions imposed for equivalence are likely to be satisfied or whether scenarios proposed by the counter-examples are likely to arise in their data. We emphasize that the assumptions below are sufficient conditions for equivalence rather than universal requirements. 

\subsection{Equivalences and key sufficient assumptions}

\subsubsection{Global functional and statistical conditional null equivalence}

There is one idealized setting in which functional and conditional statistical notions of importance coincide. This result illustrates that, under a deterministic and sufficiently regular data-generating mechanism with a non-redundant feature representation, the two perspectives identify the same null features. In practice, departures from these assumptions can cause the two notions to diverge.

\begin{lemma}[Functional--conditional statistical null equivalence]
\label{LemFuncStat}
Let $X = (X_j,X_{-j}) \sim P_X$ and suppose that
\[
Y = f(X)
\]
almost surely. Assume that:

\begin{enumerate}
    \item \textbf{Deterministic outcome:} $Y=f(X)$ almost surely.

    \item \textbf{Functional regularity:} For $P_{X_{-j}}$-almost every $x_{-j}$, the function
    \[
    x_j \mapsto f(x_j,x_{-j})
    \]
    is absolutely continuous on the support of $X_j \mid X_{-j}=x_{-j}$.

    \item \textbf{Rich conditional support:} For $P_{X_{-j}}$-almost every $x_{-j}$, the conditional distribution of $X_j$ given $X_{-j}=x_{-j}$ has a density that is positive almost everywhere on its support.

    \item \textbf{Non-redundancy:} The conditional distribution of $X_j$ given $X_{-j}$ is non-degenerate almost surely.
\end{enumerate}

Then
\[
\mathcal{N}_j^{\mathrm{func-global}}
\,\, P_X\text{-almost everywhere}
\,\,\Longleftrightarrow\,\,
\mathcal{N}_j^{\mathrm{stat-cond}}.
\]
\end{lemma}

\begin{proof}
First suppose that
\[
\frac{\partial f(x)}{\partial x_j}=0
\quad P_X\text{-almost everywhere}.
\]
By the rich conditional support assumption, this implies that
\[
\frac{\partial f(x_j,x_{-j})}{\partial x_j}=0
\]
for Lebesgue-almost every $x_j$ and $P_{X_{-j}}$-almost every $x_{-j}$. By absolute continuity, it follows that, for $P_{X_{-j}}$-almost every $x_{-j}$,
\[
f(x_j,x_{-j}) = g(x_{-j})
\]
for almost every $x_j$, for some measurable function $g$. Hence
\[
Y=g(X_{-j})
\]
almost surely. Therefore,
\[
Y \perp X_j \mid X_{-j}.
\]

Conversely, suppose that
\[
Y \perp X_j \mid X_{-j}.
\]
Since $Y=f(X)$ deterministically, for $P_{X_{-j}}$-almost every $x_{-j}$,
\[
f(X_j,x_{-j})
\]
is almost surely constant under the conditional distribution of $X_j$ given $X_{-j}=x_{-j}$. By the rich conditional support assumption, this implies that
\[
f(x_j,x_{-j}) = g(x_{-j})
\]
for Lebesgue-almost every $x_j$ and $P_{X_{-j}}$-almost every $x_{-j}$. Absolute continuity then implies
\[
\frac{\partial f(x_j,x_{-j})}{\partial x_j}=0
\]
for Lebesgue-almost every $(x_j,x_{-j})$, and hence
\[
\frac{\partial f(x)}{\partial x_j}=0
\quad P_X\text{-almost everywhere}.
\]
\end{proof}

The two key assumptions are \textbf{determinism} and \textbf{non-redundancy}. The deterministic assumption rules out outcome noise. Non-redundancy prevents conditional independence from arising simply because $X_j$ is already determined by $X_{-j}$. Conditional independence provides the standard statistical formalization of irrelevance of $X_j$ given $X_{-j}$; see, for example
~\cite{dawid1979conditional}. Under the additional deterministic and functional
regularity assumptions stated, this statistical notion can be
equated with a functional notion based on the partial derivative.

Based on the \textbf{determinism} and \textbf{non-redundancy} assumptions, the key question a practitioner must ask is whether the response exactly satisfies the functional relationship across samples (determinism) and whether the feature set can be compressed (non-redundancy) due to dependence.

\subsubsection{Risk-based and conditional statistical null equivalence}

Risk-based null importance is closely related to conditional statistical null importance, but the equivalence depends critically on the richness of the predictive function class and on the loss function. In particular, squared loss identifies the conditional mean, whereas a strictly proper scoring rule for the full conditional distribution identifies the entire conditional distribution.

\begin{lemma}[Risk--conditional statistical null equivalence]
\label{LemRiskNull}
Let $(X_j,X_{-j},Y) \sim P$. Let $\mathcal{F}$ be a class of predictive distributions for $Y$ given $X$, and let $\ell$ be a strictly proper scoring rule. Define
\[
R(f) = \mathbb{E}[\ell(Y,f(X))].
\]
Let $\mathcal{F}_{-j} \subseteq \mathcal{F}$ denote the subclass of predictors that depend only on $X_{-j}$. Assume:

\begin{enumerate}
    \item \textbf{Bayes completeness:} The true conditional distribution $P_{Y\mid X}$ belongs to $\mathcal{F}$.

    \item \textbf{Strict properness:} The population risk is uniquely minimized, up to $P_X$-almost sure equality, by the true conditional distribution $P_{Y\mid X}$.

    \item \textbf{Restriction completeness:} If
    \[
    \inf_{f\in\mathcal{F}} R(f)
    =
    \inf_{f\in\mathcal{F}_{-j}} R(f),
    \]
    then the Bayes-optimal conditional distribution admits a representation in $\mathcal{F}_{-j}$.
\end{enumerate}

Then
\[
Y \perp X_j \mid X_{-j}
\quad\Longleftrightarrow\quad
\inf_{f\in\mathcal{F}} R(f)
=
\inf_{f\in\mathcal{F}_{-j}} R(f).
\]
\end{lemma}

\begin{proof}
Suppose first that
\[
Y \perp X_j \mid X_{-j}.
\]
Then
\[
P_{Y\mid X} = P_{Y\mid X_{-j}}
\]
almost surely. Thus the Bayes-optimal conditional distribution depends only on $X_{-j}$ and belongs to $\mathcal{F}_{-j}$. Therefore,
\[
\inf_{f\in\mathcal{F}} R(f)
=
\inf_{f\in\mathcal{F}_{-j}} R(f).
\]

Conversely, suppose that
\[
\inf_{f\in\mathcal{F}} R(f)
=
\inf_{f\in\mathcal{F}_{-j}} R(f).
\]
By restriction completeness, the Bayes-optimal conditional distribution can be represented as
\[
P_{Y\mid X} = g(X_{-j})
\]
for some measurable function $g$. Hence
\[
P_{Y\mid X} = P_{Y\mid X_{-j}}
\]
almost surely, which is equivalent to
\[
Y \perp X_j \mid X_{-j}.
\]
\end{proof}

The key assumption is \textbf{function-class richness}, expressed here through Bayes completeness and restriction completeness. If $\mathcal{F}$ is too restrictive, a feature may be conditionally relevant while its contribution cannot be represented by the chosen model class.

For example, under squared loss,
\[
R(f) = \mathbb{E}\left[(Y-f(X))^2\right],
\]
the Bayes predictor is
\[
f^*(X) = \mathbb{E}[Y\mid X].
\]
Thus equality of the optimal risks with and without $X_j$ identifies only the conditional-mean null
\[
\mathbb{E}[Y\mid X]
=
\mathbb{E}[Y\mid X_{-j}],
\]
which is weaker than conditional independence. Therefore, the equivalence in Lemma~\ref{LemRiskNull} requires a loss and function class capable of identifying the full conditional distribution.

The connection between predictive risk and conditional statistical
relevance follows from the theory of proper scoring rules. In particular,
strictly proper scoring rules uniquely identify the true conditional
distribution as the Bayes-optimal predictive distribution
\citep{gneiting2007strictly}. Consequently, when the predictive function
class is sufficiently rich to contain the true conditional distribution,
equality of the optimal risks with and without $X_j$ is equivalent to
conditional independence. Thus the key question a practitioner must ask is whether for the data being used, does the response satisfies a functional relationship that belongs within the function class induced by the predictive risk.

\subsubsection{Marginal statistical and causal null equivalence under randomization}

The connection between statistical and causal null importance requires an assumption that relates the observational and interventional distributions. Randomization provides a particularly transparent setting because it is a property of the feature assignment mechanism rather than an assumption about the null itself. Let $Y^{x_j}$ denote the potential outcome under the intervention $\mathrm{do}(X_j=x_j)$.

\begin{lemma}[Marginal Statistical and Causal Null Equivalence under Randomization]
\label{LemMargCausal}
Let $Y^{x_j}$ denote the potential outcome under the intervention
$\operatorname{do}(X_j=x_j)$. Suppose the following assumptions hold:

\begin{enumerate}
    \item \textbf{Randomization (exchangeability):} For every intervention
    value $x_j$ under consideration,
    \[
    Y^{x_j}\perp X_j.
    \]

    \item \textbf{Consistency:} If $X_j=x_j$, then
    \[
    Y=Y^{x_j}
    \qquad\text{almost surely}.
    \]

    \item \textbf{Positivity:} Every intervention value $x_j$ under
    consideration lies in the support of the randomization distribution of
    $X_j$.
\end{enumerate}

Then
\[
\mathcal N_j^{\mathrm{stat-marg}}
\quad\Longleftrightarrow\quad
\mathcal N_j^{\mathrm{causal}}.
\]
Equivalently,
\[
R_j^{\mathrm{marg}}=0
\quad\Longleftrightarrow\quad
R_j^{\mathrm{causal}}=0.
\]
\end{lemma}

\begin{proof}
Under randomization and consistency, for every intervention value $x_j$
under consideration,
\[
\begin{aligned}
P(Y\in A\mid X_j=x_j)
&=
P(Y^{x_j}\in A\mid X_j=x_j) \\
&=
P(Y^{x_j}\in A)
\end{aligned}
\]
for every measurable set $A$. The first equality follows from
consistency, while the second follows from randomization. Therefore,
\[
P(Y\mid X_j=x_j)
=
P(Y^{x_j})
=
P(Y\mid \operatorname{do}(X_j=x_j)).
\]
Thus, under randomization,
\[
P(Y\mid X_j=x_j)
=
P(Y\mid \operatorname{do}(X_j=x_j))
\]
for every intervention value $x_j$ under consideration.

Suppose first that
\[
\mathcal N_j^{\mathrm{stat-marg}}
\]
holds. Then
\[
P(Y\mid X_j=x_j)
=
P(Y\mid X_j=x_j')
\]
for all $x_j,x_j'$. By the identification equality above,
\[
P(Y\mid \operatorname{do}(X_j=x_j))
=
P(Y\mid \operatorname{do}(X_j=x_j'))
\]
for all $x_j,x_j'$. Hence
\[
\mathcal N_j^{\mathrm{causal}}
\]
holds.

Conversely, suppose that
\[
\mathcal N_j^{\mathrm{causal}}
\]
holds. Then
\[
P(Y\mid \operatorname{do}(X_j=x_j))
=
P(Y\mid \operatorname{do}(X_j=x_j'))
\]
for all $x_j,x_j'$. Again using the identification equality,
\[
P(Y\mid X_j=x_j)
=
P(Y\mid X_j=x_j')
\]
for all $x_j,x_j'$. Hence
\[
\mathcal N_j^{\mathrm{stat-marg}}
\]
holds.

Therefore,
\[
\mathcal N_j^{\mathrm{stat-marg}}
\quad\Longleftrightarrow\quad
\mathcal N_j^{\mathrm{causal}},
\]
or equivalently,
\[
R_j^{\mathrm{marg}}=0
\quad\Longleftrightarrow\quad
R_j^{\mathrm{causal}}=0.
\]
\end{proof}

The key assumption is \textbf{randomization}. Randomization makes the distribution of potential outcomes independent of the assigned value of $X_j$, thereby allowing the observational conditional distribution to identify the corresponding interventional distribution. Consistency connects the observed outcome to the relevant potential outcome, while positivity ensures that the interventions of interest are identifiable.

The equivalence between the statistical and causal nulls under
randomization follows from the standard potential-outcomes framework.
Randomization implies exchangeability of the potential outcomes and
treatment assignment, while consistency links the observed outcome to
the corresponding potential outcome; see~\cite{rubin1974estimating,hernan2023causal}.

\subsubsection{Marginal and conditional statistical nulls}

A natural question is whether marginal and conditional statistical null importance coincide when the features are independent. Feature independence is sufficient for one direction, but not for the reverse direction in general.

\begin{lemma}[Feature independence and marginal--conditional nulls]
\label{LemMargCond}
Suppose that
\[
X_j \perp X_{-j}.
\]
Then
\[
Y \perp X_j\mid X_{-j}
\quad\Longrightarrow\quad
Y \perp X_j.
\]
The converse does not hold in general.
\end{lemma}

\begin{proof}
Suppose that
\[
Y \perp X_j\mid X_{-j}.
\]
Then
\begin{align*}
&P(Y\in A\mid X_j=x_j,X_{-j}=x_{-j})\\
=\,\,
&P(Y\in A\mid X_{-j}=x_{-j})
\end{align*}
for every measurable set $A$. Therefore,
\[
\begin{aligned}
&P(Y\in A\mid X_j=x_j)\\
=
&\int
P(Y\in A\mid X_j=x_j,X_{-j}=x_{-j}) dP_{X_{-j}\mid X_j=x_j}(x_{-j}) \\
=
&\int
P(Y\in A\mid X_{-j}=x_{-j})
\,dP_{X_{-j}}(x_{-j}),
\end{aligned}
\]
where the second equality follows from
\[
X_j\perp X_{-j}.
\]
The final expression does not depend on $x_j$. Hence
\[
Y\perp X_j.
\]
\end{proof}

The converse does not hold even when the features are independent. For example, let
\[
X_j,X_k\stackrel{\mathrm{ind}}{\sim}\operatorname{Bernoulli}(1/2)
\]
and define
\[
Y=X_j\oplus X_k,
\]
where $\oplus$ denotes the XOR operation. Then
\[
Y\perp X_j,
\]
but
\[
Y\not\perp X_j\mid X_k.
\]
Thus \textbf{feature independence} alone is not sufficient for equivalence of marginal and conditional statistical nulls. An additional \textbf{no-interaction} or \textbf{structural homogeneity} assumption is required for the reverse implication.

\subsection{Counterexamples: Separation of Null Importance Notions}

The previous section established conditions under which specific notions of null importance coincide. In this section, we construct several counterexamples illustrating strict separations between functional, statistical, risk-based, and causal notions of null importance.

\paragraph{Marginal vs Conditional Statistical Null}

Perhaps one of the most widely known difference in the statistics literature is the distinction between marginal and conditional null (see e.g.~\cite{dawid1979conditional}).

\begin{example}[Marginal null does not imply conditional null]
Let $X_1 \sim \mathrm{Bernoulli}(1/2)$ and define
\[
X_2 = X_1 \oplus \epsilon, \quad \epsilon \sim \mathrm{Bernoulli}(1/2),
\]
independent of $X_1$. Let $Y = X_1$.

Then:
\begin{itemize}
\item $I(X_2; Y) = 0$, so $X_2$ is marginally statistically null.
\item However, $I(X_2; Y \mid X_1) > 0$, so $X_2$ is not conditionally statistically null.
\end{itemize}

Thus, marginal and conditional null importance do not coincide in general.
\end{example}

\paragraph{Functional Null vs Conditional Statistical Null}
We provided an equivalence result in Lemma \ref{LemFuncStat} and also discussed conditions under which this equivalence breaks down. To flesh this out further, we provide two more examples where the Lemma's assumptions fail to hold.

\begin{example}[Redundant representation]
Let
\[
X_1 \sim \mathcal N(0,1),
\qquad
X_2 = X_1,
\]
and define
\[
Y = X_2.
\]

The prediction function can be written as
\[
f(x_1,x_2)=x_2,
\]
which does not depend on \(x_1\). Therefore there exists a function
\[
g(x_2)=x_2
\]
such that
\[
f(x_1,x_2)=g(x_2),
\]
and hence \(X_1\) satisfies the functional null property.

Moreover,
\[
Y=X_2,
\]
implies
\[
P(Y\mid X_1,X_2)
=
P(Y\mid X_2),
\]
and therefore
\[
Y \perp X_1 \mid X_2.
\]

Thus \(X_1\) is both functionally null and conditionally statistically null. This example illustrates how redundancy can create multiple equivalent representations of the prediction mechanism. Although the two notions of null importance coincide in this setting, the relevance of individual features becomes non-identifiable because the same predictive information is duplicated across multiple variables.
\end{example}

\begin{example}[Conditional statistical null without functional null]
\label{ex:conditional_not_functional}
Let
\[
X_1 \sim \mathcal{N}(0,1),
\qquad
X_2 = X_1,
\]
and define
\[
Y = \frac{1}{2}X_1 + \frac{1}{2}X_2.
\]

Since \(X_1 = X_2\), we have
\[
Y = X_2 = X_1.
\]

Consequently,
\[
P(Y \mid X_1,X_2)
=
P(Y \mid X_2),
\]
so that
\[
Y \perp X_1 \mid X_2.
\]
Similarly,
\[
Y \perp X_2 \mid X_1.
\]
Thus each variable satisfies the conditional statistical null property.

Now consider the prediction function
\[
f(x_1,x_2)
=
\frac{1}{2}x_1+\frac{1}{2}x_2.
\]
There does not exist a function \(g\) such that
\[
f(x_1,x_2)=g(x_2)
\qquad
\text{for all } (x_1,x_2)\in\mathbb{R}^2,
\]
since, for example,
\[
f(0,0)=0,
\qquad
f(2,0)=1,
\]
while both inputs have the same value of \(x_2\). Therefore \(X_1\) is not functionally null. By symmetry, \(X_2\) is also not functionally null.

This example illustrates that the conditional statistical null and functional null need not coincide due to redundancy in the feature set which is assumed away in Lemma~\ref{LemFuncStat}. The discrepancy arises because perfect redundancy makes the prediction mechanism statistically non-identifiable: either feature alone determines the response under the observational distribution, even though the chosen functional representation explicitly depends on both variables.
\end{example}

\paragraph{Risk-Based Null vs Conditional Statistical Null}

\begin{example}[Failure of Bayes-completeness]
\label{ex:failure_bayes}
Let
\[
X_1 \sim \mathrm{Unif}(-1,1),
\qquad
X_2 = X_1^2,
\]
and define
\[
Y = X_2 = X_1^2.
\]

Since \(X_2\) is a deterministic function of \(X_1\),
\[
P(Y \mid X_1,X_2)
=
P(Y \mid X_1),
\]
and therefore
\[
Y \perp X_2 \mid X_1.
\]
Thus \(X_2\) satisfies the conditional statistical null property.

Now consider the hypothesis class of linear predictors
\[
\mathcal{F}
=
\left\{
f(x_1,x_2)
=
\beta_0+\beta_1x_1+\beta_2x_2
:
\beta_0,\beta_1,\beta_2\in\mathbb{R}
\right\}.
\]

When both variables are available, the predictor
\[
f(x_1,x_2)=x_2
\]
belongs to \(\mathcal{F}\), so the optimal risk equals the Bayes risk,
\[
R^*=0.
\]

If \(X_2\) is removed, however, the restricted hypothesis class becomes
\[
\mathcal{F}_{-2}
=
\left\{
f(x_1)
=
\beta_0+\beta_1x_1
:
\beta_0,\beta_1\in\mathbb{R}
\right\},
\]
which cannot represent the quadratic function
\[
Y=X_1^2.
\]
Consequently,
\[
R^*_{-2}>R^*,
\]
so \(X_2\) is not risk-null.

\end{example}
This example demonstrates that conditional statistical null and risk null need not coincide when the hypothesis class is not Bayes-complete. Although \(X_2\) contains no information beyond \(X_1\), its explicit inclusion enables the restricted hypothesis class to represent the regression function exactly. In contrast, a Bayes-complete hypothesis class would allow the same predictor to be represented using \(X_1\) alone, making \(X_2\) risk-null as well. This nonlinear example is recurring scenario where methods suited to linear and nonlinear relationships differ.

\paragraph{Causal null vs Conditional Statistical Null}

\begin{example}[Causal importance without conditional statistical importance]
Let
\[
X \sim \mathcal{N}(0,1),
\]
and define the structural equations
\[
M = X,
\qquad
Y = M.
\]

The causal graph is
\[
X \longrightarrow M \longrightarrow Y.
\]

Since intervening on \(X\) changes \(M\), which in turn changes \(Y\), the variable \(X\) is not causally null.

On the other hand,
\[
P(Y\mid X,M)
=
P(Y\mid M),
\]
because \(Y\) depends only on the mediator \(M\). Hence
\[
Y \perp X \mid M,
\]
so \(X\) is conditionally statistically null.
\end{example}
This example illustrates that conditional statistical null does not imply causal null. Conditioning on a mediator blocks the indirect causal pathway, rendering the treatment statistically redundant for prediction even though it remains causally relevant.

\section{Characterizing feature analysis methods by null importance}
\label{sec}

The notions of null importance given in Section \ref{sec:null_notions} allow us to frame precise scientific questions about the role of a variable in a dataset, and Section \ref{sec:equivalences_and_counterexamples} gave conditions under which they coincide. We now turn to methods for establishing whether features are null. For each feature-analysis method, we ask two questions: (i) which notion of null importance does the method directly target, and (ii) does the population importance statistic produced by the method identify that null?
This distinction is important. A method may be motivated by a particular notion of relevance without its zero importance statistic identifying that notion. Further, a method may identify a null only under assumptions that connect its statistical target to its scientific context. We therefore distinguish between \emph{targeting} a null notion and \emph{identifying} that null. 

The resulting taxonomy is summarized in Table~\ref{tab:null_methods}. The table should be read as a guide to methodological alignment rather than as a ranking of methods. The appropriate choice depends on the scientific question, the available data, and the assumptions that are credible in the application. Since there are so many feature analysis methods, we focus on specific examples of feature attribution, feature selection and functional sensitivity methods. Methods not included may also be put through this taxonomy or framework.
This illustrates another key motivation of null importance which is to evaluate feature analysis methods. A method should be evaluated against the null notion it is intended to answer, not merely by the numerical behavior of its importance scores.

\subsection{Feature attribution methods}
\label{sec:feature_attribution_methods}

We begin with feature attribution methods, which assign to each feature an importance score $\phi_j$ intended to quantify its contribution to a specified value functional $V: 2^{p} \rightarrow \mathbb{R}$.

Shapley values provide a general axiomatic framework for such attribution (see e.g.~\cite{Shapley1953Value, LundbergLee2017SHAP, LundbergErionChen2019TreeSHAP, HeskesEtAl2020CausalShapley, FryeStrumke2019ShapleyRisk, CovertLundbergLee2020SAGE, CovertLee2021KernelSHAP}). Given a value functional $V$, the Shapley value for feature $j$ is defined as
\begin{equation*}
\phi_{j} =
\sum_{S\subseteq \{1,\ldots,p\} \setminus \{j\}}
\frac{1}{p\binom{p-1}{|S|}}
\left[
V(S \cup \{j\}) - V(S)
\right].
\end{equation*}

The Shapley values are uniquely characterized by the following axioms:

\begin{itemize}
    \item \textbf{Efficiency:} $\sum_{j=1}^p \phi_{j} = V(\{1,\ldots,p\}) - V(\emptyset)$.
    \item \textbf{Symmetry:} if two features contribute identically to $V$, then they receive equal attribution.
    \item \textbf{Null player:} if $V(S \cup \{j\}) = V(S)$ for all $S$, then $\phi_j = 0$.
    \item \textbf{Linearity:} Shapley values are linear in $V$.
\end{itemize}

The null player axiom suggests a connection to null importance; however, it is a uni-directional statement. A zero value for $\varphi_{j}$ does not imply zero contribution to $V$. Further, a key feature of Shapley values is that the interpretation of null importance depends entirely on the choice of value functional $V$.

\begin{itemize}
    \item $V(S) = \mathbb{E}[f(X) \mid \mathrm{do}(X_S = x_S)]$ corresponds to causal Shapley values.
    \item $V(S) = \mathbb{E}[f(X) \mid X_S = x_S]$ corresponds to observational/statistical Shapley values.
    \item $V(S) = \mathcal{R}(f_S)$ corresponds to risk-based Shapley values.
    \item $V(S) = f(x_S, x'_S)$ corresponds to functional/baseline attribution (e.g., Integrated Gradients-type constructions).
\end{itemize}

Each of these choices induces a different notion of relevance and therefore a different notion of null importance. We now analyze whether Shapley values satisfy strong or weak null properties under these different choices.

\paragraph{Risk-based Shapley}

For function $f$ we define the difference of risk:
\[
\mathcal{R}(f(X)) - \mathcal{R}(f_{-j}(X_{-j})).
\]

The corresponding Shapley value~\citep{CovertLundbergLee2020SAGE, FryeStrumke2019ShapleyRisk} is
\begin{align*}
\phi_{j}^{\mathrm{risk}} =
\sum_{S\subseteq [p]\setminus\{j\}}
\frac{1}{p\binom{p-1}{|S|}}
\Big[ &\mathcal{R}(f_S(X_S)) - \mathcal{R}(f_{S \cup j}(X_{S \cup j})) \Big].
\end{align*}
The \emph{targeted null} is a risk-based null importance (not identical to the original definition since $f$ is pre-defined):
\[
R_{-j} = R.
\]
However, in general Shapley values \emph{do not} identify the risk-based null. Indeed, the risk null concerns the marginal contribution of $X_j$ when all other features are available, whereas the Shapley value averages marginal contributions over all subsets. A feature can therefore be risk-null in the full model while making nonzero contributions to smaller coalitions.

For example, consider a chain
\[
X_1\longrightarrow X_2\longrightarrow X_3\longrightarrow Y,
\]
with
\[
X_2=X_1+\gamma,\qquad
X_3=X_2+\delta,\qquad
Y=X_3+\epsilon,
\]
where the noise variables are mutually independent. Once $X_3$ is available, $X_1$ provides no additional predictive information. Thus $X_1$ is risk-null relative to $X_3$. Nevertheless, $X_1$ contributes predictive information to coalitions that do not contain $X_3$, and therefore receives a nonzero Shapley attribution. The example illustrates the fundamental distinction between \emph{full-model leave-one-feature-out importance} and \emph{average coalition contribution}.

Consequently, risk-based Shapley should not be interpreted as a general test for risk-based null importance.

\paragraph{Causal Shapley}

Causal Shapley values \citep{HeskesEtAl2020CausalShapley, JanzingMinoricsBlobaum2020CausalFeatureRelevance} replace the observational distribution in the Shapley value function by an interventional distribution. In the \(do\)-based formulation, the value of a coalition \(S\) is
$$
V(S;x_S)
=
\mathbb{E}\!\left[
f(X)\mid \operatorname{do}(X_S=x_S)
\right].
$$

Thus, the corresponding Shapley value is

\begin{align*}
\phi_j^{\mathrm{causal}}(x)
&=
\sum_{S\subseteq -j}
\frac{1}{p\binom{p-1}{|S|}}
\Big[
V(S\cup\{j\};x_{S\cup\{j\}})
\nonumber\\
&\qquad\qquad\qquad
- V(S;x_S)
\Big].
\end{align*}

This construction targets an \emph{interventional functional} notion of relevance: feature \(X_j\) is causally null for the prediction function if intervening on \(X_j\) does not change the corresponding interventional expectation. In particular, define a variant of the causal null defined through expectation:
$$
\mathbb{E}\!\left[
f(X)\mid \operatorname{do}(X_j=x_j)
\right]
=
\mathbb{E}\!\left[
f(X)\mid \operatorname{do}(X_j=x_j')
\right]
$$
for all intervention values \(x_j,x_j'\) in the relevant support. The null-player property of the Shapley value gives the following one-way guarantee.

\begin{lemma}[Causal null implies zero causal Shapley]
\label{lem:causal_shapley_null}
Suppose that \(X_j\) is causally null in the sense that its intervention does not change the interventional value function for any coalition of the remaining features. Then
\begin{align*}
&\mathbb{E}\!\left[
f(X)\mid \operatorname{do}(X_j=x_j)
\right]
\,\,=\,\,
\mathbb{E}\!\left[
f(X)\mid \operatorname{do}(X_j=x_j')
\right]\\
&\Longrightarrow\quad
\phi_j^{\mathrm{causal}}=0.
\end{align*}

\end{lemma}

\begin{proof}
If \(X_j\) is causally null, then for every \(S\subseteq \{1, \dots, p\} -j\),
$$
V(S\cup\{j\};x_{S\cup\{j\}})
=
V(S;x_S).
$$

Hence every marginal contribution of \(X_j\) is zero. Since the Shapley value is a weighted average of these marginal contributions,
$$
\phi_j^{\mathrm{causal}}=0.
$$
\end{proof}

The converse does not hold in general. Unlike a value function based on optimal predictive risk, the interventional expectation \(V(S;x_S)\) is not monotone in the set of intervened variables. Consequently, the marginal contributions
$$
\Delta_j(S)
=
V(S\cup\{j\};x_{S\cup\{j\}})
-
V(S;x_S)
$$

may be positive for some coalitions and negative for others. The Shapley value averages these contributions, so they can cancel:
$$
\sum_{S\subseteq -j}w_S\Delta_j(S)=0
$$

even though \(\Delta_j(S)\neq 0\) for some \(S\). Thus
\begin{align*}
\phi_j^{\mathrm{causal}}&=0
\quad\centernot\Longrightarrow\quad\\
\mathbb{E}\!\left[
f(X)\mid \operatorname{do}(X_j=x_j)
\right]
&\,\,=\,\,
\mathbb{E}\!\left[
f(X)\mid \operatorname{do}(X_j=x_j')
\right]
\end{align*}
without an additional assumption preventing such cancellation.

Importantly, this is an \emph{identification} issue rather than a claim that cancellation is common in applications. The causal Shapley literature of \cite{HeskesEtAl2020CausalShapley} and \cite{JanzingMinoricsBlobaum2020CausalFeatureRelevance} establishes the interventional construction and its causal interpretation, but does not, to our knowledge, provide a general no-cancellation condition or establish that cancellation is empirically frequent.

\paragraph{Integrated Gradients}

Integrated Gradients (IG) \citep{SundararajanTalyYan2017IG} is a path-based attribution method that measures the sensitivity of a prediction function relative to a baseline \(x'\). For the straight-line path from \(x'\) to \(x\), the attribution to feature \(X_j\) is
\begin{align}
\label{eq:integrated_gradients}
\operatorname{IG}_j(x;x')
=
(x_j-x_j')
\int_0^1
\frac{\partial f(x'+\alpha(x-x'))}
{\partial x_j}
\,d\alpha.
\end{align}

The method therefore targets \emph{functional} relevance. Its attribution is determined entirely by the prediction function \(f\), the input \(x\), and the chosen baseline \(x'\), rather than by the distribution of \((X,Y)\).

\begin{lemma}[Functional null implies zero Integrated Gradients]
\label{lem:ig_functional_null}
Suppose \(f\) is differentiable almost everywhere and \(X_j\) is functionally null (along the path from $x$ to $x'$). Then, for every input \(x\) and baseline \(x'\) for which the Integrated Gradients path is defined,
$$
\operatorname{IG}_j(x;x')=0.
$$

Consequently, any population-level aggregation of Integrated Gradients that averages or otherwise aggregates the absolute attributions also assigns zero importance to \(X_j\).
\end{lemma}

\begin{proof}
Since \(f(x_j,x_{-j})=g(x_{-j})\), the prediction function is constant with respect to \(x_j\). Hence
$$
\frac{\partial f(x)}{\partial x_j}=0
$$

almost everywhere. Therefore,
$$
\frac{\partial f(x'+\alpha(x-x'))}{\partial x_j}
=0
$$

for almost every \(\alpha\in[0,1]\). Substituting into the definition of Integrated Gradients gives
$$
\operatorname{IG}_j(x;x')
=
(x_j-x_j')
\int_0^1 0\,d\alpha
=0.
$$

\end{proof}
This result is the functional-null property underlying the \emph{Sensitivity(b)} axiom of \cite{SundararajanTalyYan2017IG}: if the prediction function is constant in a feature, that feature receives zero attribution. Sundararajan et al. also establish that Integrated Gradients satisfies Sensitivity (a), which requires a nonzero attribution when the input and baseline differ only in feature \(j\) and the corresponding predictions differ. Thus IG is designed to avoid the failure of ordinary local gradients to detect changes that occur away from the endpoint of the attribution path~\cite{SundararajanTalyYan2017IG}.

The result above is a \emph{null guarantee} in the sense functional nullity along a path implies zero Integrated Gradients. The converse, however, is not true. Indeed,
$$
\operatorname{IG}_j(x;x')=0
$$

may occur even when \(f\) depends on \(x_j\), because the signed derivatives along the path can cancel. Therefore,
$$
\mathcal N_j^{\mathrm{func-path}}
\quad\Longrightarrow\quad
\operatorname{IG}_j(x;x')=0,
$$

but, in general,
$$
\operatorname{IG}_j(x;x')=0
\quad\centernot\Longrightarrow\quad
\mathcal N_j^{\mathrm{func-path}}.
$$

\subsection{Feature selection methods}

We now classify standard feature selection methods according to which notion of null importance they satisfy.

\paragraph{Correlation and mutual information screening}
Correlation screening~\cite{FanLv2008SIS, FanSong2010SISGLM} selects features based on marginal association with the response. The \emph{targeted null} is marginal statistical association, with the precise null depending on the association measure. It satisfies:
\[
\mbox{Corr}(X_j, Y) = 0 \;\Leftarrow\; X_j \perp Y,
\]
However the converse does not hold, since correlation captures linear dependence only. The simple univariate example $Y = X^2$ where $X \sim \mbox{Normal}(0,1)$ shows that $\mbox{Corr}(X, Y) = 0$ but $X \not \perp Y$.

The weaknesses of correlation screening are addressed by mutual information screening, where $\mbox{Corr}(X_j, Y)$ is replaced by mutual information $I(X_j;Y)$. For mutual information screening~\cite{Battiti1994MutualInformation, PengLongDing2005mRMR}, the targeted null is marginal statistical since
$$
I(X_j;Y)=0
\quad\Longleftrightarrow\quad
X_j\perp Y
$$
under the standard definition of mutual information. Thus mutual information is an exact population measure for the marginal statistical null.

\paragraph{Lasso}
Under a correctly specified linear model with sparsity assumptions, Lasso~\cite{Tibshirani1996Lasso} recovers features satisfying:
\[
Y \perp X_j \mid X_{-j}
\]
in the population limit.

Thus, under standard assumptions (restricted eigenvalue, irrepresentability), Lasso satisfies a form of \textbf{conditional null consistency}. However, this equivalence fails under model misspecification or nonlinear dependence.

\paragraph{Knockoffs}

Model-X knockoffs \citep{BarberCandes2015Knockoffs, CandèsFanJansonLv2018Knockoffs} target the 
\emph{conditional statistical null} $\mathcal N_j^{\mathrm{stat-cond}}$ (see Defn.~\ref{defn:cond_stat}). The key theoretical result is that, under the Model-X assumptions that the knockoff variables $\widetilde X$ satisfy
$$
(X,\widetilde X)_{\operatorname{swap}(S)}
\overset{d}{=}
(X,\widetilde X)
\quad\text{for all }S,
\qquad
\widetilde X\perp Y\mid X,
$$
the knockoff statistic for a conditionally null feature has a symmetric distribution about zero. This sign symmetry yields finite-sample false discovery rate control:
$$
\operatorname{FDR}\leq q
$$
for the knockoff selection procedure at target level \(q\) \citep{CandèsFanJansonLv2018Knockoffs}. Thus, under valid knockoff construction, the method provides a finite-sample guarantee for identifying features that are non-null with respect to the conditional statistical null. Unlike methods based on a population importance score, however, this is a selection and error-control guarantee rather than a statement that a zero statistic identifies the null.

\paragraph{Leave-One-Covariate-Out (LOCO)}

Leave-One-Covariate-Out (LOCO) importance \citep{LeiZhuWasserman2018LOCO} targets $\mathcal N_j^{\mathrm{risk}}$ (Def~\ref{defn:risknull}). The corresponding population LOCO importance is
$$
\Delta_j^{\mathrm{LOCO}}
=
R^*_{-j}-R^*.
$$
Then by definition,
$$
\Delta_j^{\mathrm{LOCO}}=0
\quad\Longleftrightarrow\quad
\mathcal N_j^{\mathrm{risk}}.
$$
Thus LOCO directly identifies its targeted risk-based null at the population level. The interpretation remains dependent on the predictive function class \(\mathcal F\): risk nullity means that \(X_j\) can be removed without increasing the optimal achievable risk within that class, and does not in general imply marginal or conditional statistical nullity. Lei et al.~\cite{LeiZhuWasserman2018LOCO} develop distribution-free inference for LOCO importance, providing statistical guarantees for estimating this population quantity.

\paragraph{Generalized Covariance Measure (GCM)}

The generalized covariance measure (GCM) \citep{ShahPeters2020GCM} targets the \emph{conditional statistical null}
$$
\mathcal N_j^{\mathrm{cond}}:
\qquad
Y\perp X_j\mid X_{-j}.
$$

For scalar \(X_j\) and \(Y\), define the conditional-mean residuals
$$
\varepsilon_j
=
X_j-\mathbb E[X_j\mid X_{-j}],
\qquad
\varepsilon_Y
=
Y-\mathbb E[Y\mid X_{-j}],
$$

and the population generalized covariance measure
$$
\operatorname{GCM}_j
=
\mathbb E[\varepsilon_j\varepsilon_Y].
$$

If
$$
Y\perp X_j\mid X_{-j},
$$

then
$$
\operatorname{GCM}_j=0,
$$

provided the relevant second moments exist. Indeed, conditional independence implies
$$
\mathbb E[\varepsilon_Y\mid X_j,X_{-j}]
=
\mathbb E[Y\mid X_j,X_{-j}]
-
\mathbb E[Y\mid X_{-j}]
=0,
$$

and hence
$$
\mathbb E[\varepsilon_j\varepsilon_Y]
=
\mathbb E\!\left[
\varepsilon_j
\mathbb E[\varepsilon_Y\mid X_j,X_{-j}]
\right]
=0.
$$

Thus GCM provides a valid \emph{null guarantee}: conditional statistical nullity implies zero population GCM. The converse does not hold in general. A zero covariance of the residuals only rules out conditional linear association between the residualized variables (see the example for correlation screening); conditional dependence can remain through nonlinear or higher-order relationships. Shah and Peters \citep{ShahPeters2020GCM} therefore formulate GCM as a conditional-independence test whose validity relies on consistently estimating the conditional means. 

\paragraph{MinShap}
A recent method MinShap \citep{RaskuttiMinShap} targets the \emph{statistical conditional and risk-based null}: feature \(X_j\) is null if its contribution remains zero after conditioning on all other features. Let
$$
VI_j^S
=
V(S\cup\{j\})-V(S)
$$

denote the marginal value of \(X_j\) given \(X_S\). The MinShap population importance is defined by
$$
\operatorname{MinShap}_j
=
\min_{S\subseteq [p]\setminus\{j\}} VI_j^S,
$$
which simply replaces the average operator by the minimum operator for the statistics $VI_j^S$. The key theoretical result is that, under \emph{null monotonicity},
$$
S\subseteq T,
\qquad
VI_j^S=0
\quad\Longrightarrow\quad
VI_j^T=0,
$$
we have
$$
\operatorname{MinShap}_j=0
\quad\Longleftrightarrow\quad
VI_j^{-j}=0.
$$
Thus, under null monotonicity, MinShap identifies the conditional risk-based null~\cite{RaskuttiMinShap}: zero MinShap importance is equivalent to zero marginal contribution after conditioning on all remaining features.

For the population predictive-risk value functional, null monotonicity holds under suitable faithfulness and no-reverse-causation conditions. In contrast to ordinary Shapley values, which average marginal contributions over subsets and can therefore obscure conditional nullity, MinShap is constructed to preserve this null property.

\paragraph{Mean Decrease Impurity (MDI)}

Mean decrease impurity (MDI) is a feature importance measure derived from tree-based models~\cite{Breiman2001RandomForests}. For a fitted tree, the MDI importance of feature \(X_j\) is the total decrease in the splitting criterion attributable to splits on \(X_j\), typically averaged over the trees in an ensemble.

MDI does not target a single population notion of null importance. In particular, because tree construction is based on greedy impurity reduction, MDI can reflect marginal association, conditional association, feature redundancy, and algorithm-specific choices such as split selection and randomization. Consequently, MDI does not in general provide a null-importance guarantee for either marginal or conditional statistical relevance.

\begin{example}[MDI and conditional null importance]
Consider binary features
$$
X_1\sim\operatorname{Bernoulli}(1/2),
\qquad
X_2=X_1,
$$

and response
$$
Y=X_1.
$$

Then \(X_2\) is conditionally statistically null given \(X_1\):
$$
Y\perp X_2\mid X_1.
$$

However, \(X_1\) and \(X_2\) are observationally indistinguishable predictors of \(Y\). At a node where both variables are available, they produce the same impurity reduction. A tree-building algorithm must therefore resolve a tie through its implementation-specific tie-breaking or randomization mechanism.

In a randomized tree ensemble, different trees can select \(X_1\) or \(X_2\) for the corresponding split. Consequently, both variables can receive positive MDI importance,
$$
\operatorname{MDI}(X_1)>0,
\qquad
\operatorname{MDI}(X_2)>0,
$$

even though
$$
Y\perp X_2\mid X_1.
$$
Thus MDI does not satisfy conditional null importance. The example also illustrates that MDI is sensitive to the representation of redundant predictive information: the allocation of importance between statistically indistinguishable features can depend on the tree-building procedure rather than on an intrinsic population property.
\end{example}

This phenomenon is consistent with the theoretical analysis of tree-based variable importance, which shows that impurity-based importance can be affected by feature distributions, selection mechanisms, and correlations among predictors \cite{StroblBoulesteixZeileisHothorn2008ConditionalVariableImportance,StroblBoulesteixKneibAugustin2008RandomForestBias}. In particular, Strobl et al.~\citep{StroblBoulesteixZeileisHothorn2008ConditionalVariableImportance} demonstrate that the variable-selection mechanism underlying random forests can favor certain variables even when their predictive relevance is comparable.

\subsection{Functional sensitivity methods}

These methods characterize how a fixed prediction function \(f\) depends on a feature rather than testing whether the feature can be removed from a predictive model. Their targeted nulls are therefore functional, but the summaries they use differ in whether they capture local sensitivity, marginal effects, or total functional dependence~\cite{Apley2020CPI, ApleyZhu2020ALE, SimonyanVedaldiZisserman2014Saliency, SmilkovEtAl2017SmoothGrad, GoldsteinEtAl2015ICE, ShrikumarGreensideKundaje2017DeepLIFT}.

\paragraph{Saliency maps}

Saliency maps target the \emph{local functional null}
$$
\mathcal N_j^{\mathrm{func-local}}:
\qquad
\frac{\partial f(x)}{\partial x_j}=0.
$$

The classical saliency map \citep{SimonyanVedaldiZisserman2014Saliency} uses
$$
S_j(x)=
\left|
\frac{\partial f(x)}{\partial x_j}
\right|.
$$

Thus,
$$
\mathcal N_j^{\mathrm{func-local}}
\quad\Longrightarrow\quad
S_j(x)=0
$$

at every point where the derivative exists. Under standard regularity and support conditions, if the derivative is zero throughout the relevant support, then \(f\) is constant in \(x_j\), so the feature is functionally null. The converse can therefore be stated as a functional result rather than as a statistical or predictive one.

Saliency maps are consequently sensitive to local behavior: a feature can have zero gradient at a particular observation even though the prediction function depends on that feature elsewhere. This is one motivation for path-based methods such as Integrated Gradients \citep{SimonyanVedaldiZisserman2014Saliency}.

\paragraph{Partial Dependence}

Partial dependence (PD) \citep{Friedman2001GreedyFunction} targets a \emph{marginal functional effect} rather than the full functional null. Its population partial dependence function is
$$
PD_j(x_j)
=
\mathbb E_{X_{-j}}
\left[
f(x_j,X_{-j})
\right].
$$

A natural null for PD is therefore
$$
\mathcal N_j^{\mathrm{PD}}:
\qquad
PD_j(x_j)=c
\quad\text{for all }x_j,
$$

for some constant \(c\). Hence a flat population PD curve implies zero PD effect. However, a flat PD curve does \emph{not} imply the global functional null $\mathcal{N}_{j}^{\mathrm{func-global}}$ (see Definition \ref{defn:global_functional}). For example, let $X_j$ and $X_k$ be independent with
$\mathbb E[X_k]=0$, and let
$$
f(x_j,x_k)=x_jx_k.
$$

Then
$$
PD_j(x_j)
=
x_j\mathbb E[X_k]
=0
$$

for every \(x_j\), even though \(f\) depends directly on \(x_j\). Thus PD can fail to identify functional relevance because averaging can cancel heterogeneous effects. This distinction is inherent in the marginal averaging construction of PD \citep{Friedman2001GreedyFunction}.

\paragraph{Accumulated Local Effects}

Accumulated Local Effects (ALE) \citep{ApleyZhu2020ALE} also targets a \emph{marginal functional effect}, but averages local derivatives over the observed conditional distribution rather than averaging predictions over the marginal distribution of the other features. For a continuous feature,
$$
ALE_j(x_j)
=
\int_{z_0}^{x_j}
\mathbb E
\left[
\frac{\partial f(z,X_{-j})}{\partial z}
\Bigm|X_j=z
\right]dz.
$$

Consequently, if $\mathcal{N}_{j}^{\mathrm{func-global}}$ holds, then
$$
\frac{\partial f(x)}{\partial x_j}=0
$$

and therefore
$$
ALE_j(x_j)=\text{constant}.
$$

Thus the global functional null implies a flat ALE curve.

The converse does not hold in general: the conditional averaging of local derivatives can cancel across values of \(X_{-j}\), so a flat ALE curve need not imply that \(f\) is globally independent of \(x_j\). ALE was introduced in part to avoid the extrapolation problems of PD when predictors are correlated; it does not turn a marginal functional summary into a complete characterization of functional dependence \citep{ApleyZhu2020ALE}.

A recurring theme is that aggregation can obscure null identification. Shapley values average marginal contributions across coalitions, Integrated Gradients integrates signed derivatives along a path, and PD and ALE average predictions or local derivatives over a marginal or conditional distribution; in each case, nonzero effects in some coalitions, path segments, or conditioning sets can offset one another in the aggregate. This is not inherently a defect of these summaries, but it suggests that methods intended to identify null importance should be evaluated for whether their aggregation step can erase relevant signal.

\begin{table*}[!t]
\centering
\caption{Targeted null importance and theoretical guarantees for representative
feature analysis methods. Results refer to population quantities unless
otherwise indicated.}
\label{tab:null_methods}
\small
\begin{tabularx}{\linewidth}{@{}l@{\hspace{3.5em}}l@{\hspace{3.5em}}X@{\hspace{3.5em}}X@{}}
\textbf{Method} &
\textbf{Targeted null} &
\textbf{Theoretical guarantee} &
\textbf{Key condition} \\
\midrule
Correlation screening &
Marginal statistical &
One-way null guarantee &
Linear relationship \\

Mutual information &
Marginal statistical &
Exact identification &
$X_j\perp Y \iff I(X_j;Y)=0$ \\

MDI &
Unclear &
\textbf{Counterexample} &
Redundant features \\

Knockoffs &
Conditional statistical &
FDR guarantee &
Valid knockoff construction \\

GCM &
Conditional statistical &
One-way null guarantee &
Consistent conditional-mean \\

LOCO &
Risk-based &
Exact identification &
Population risk \\

Lasso &
Linear-model relevance &
Support recovery &
Correct sparse linear model \\

Risk-based Shapley &
Risk-based &
\textbf{Counterexample} &
Shapley averaging \\

Causal Shapley &
Causal &
One-way null guarantee &
Intervention consistency \\

MinShap &
Conditional risk-based &
Exact identification &
Null monotonicity \\

Integrated Gradients &
Path-based functional &
One-way null guarantee &
Differentiability and path \\

Saliency maps &
Local functional &
One-way null guarantee &
Differentiability \\

Partial Dependence &
Marginal functional &
Exact identification &
Population PD \\

ALE &
Marginal functional &
One-way null guarantee &
Regularity and support \\
\bottomrule
\end{tabularx}
\end{table*}
\section{Case studies in null importance}

\subsection{Counterexamples and Equivalence on Synthetic Data}
\label{sec:synthetic_examples}

We next use synthetic data to distinguish the notions of null importance defined in Section \ref{sec:null_notions}. We consider ten data-generating mechanisms, each with $p = 18$ features, six deliberately structured features $\mathcal{S}$, detailed in Table \ref{tab:mean_functions}, and twelve independent $N\left(0, 1\right)$ variables.
For each mechanism, we generate seven replicates at each of three sample sizes $n \in \{50, 500, 5000\}$, and each mechanism has an underlying response signal $f$. We generate both continuous responses,
$$y \sim N\left(f\left(x\right), \sigma_{y}^{2}\right)$$ 
and binary responses,
$$y \sim \text{Bernoulli}(\text{logit}^{-1}\left(f\left(x\right)\right)).$$
Mechanism D1 is an additive linear model, D2 remains additive but is quadratic. D3 and D4 introduce interactions. D4 reduces to XOR when the group sizes $G_{k}$ are set to two. D5 - D7 use simple structural causal models. In D5, each nonnull variable is matched by a highly correlated proxy which is irrelevant for prediction. D6 introduces an unobserved confounder, and D7 uses two length-three chain graphs with all mediators measured, so that conditioning on downstream mediators makes upstream variables conditionally null. D8 - D10 distinguish between functional, conditional, and risk-based null importance. In D8, features appear in highly correlated pairs, both of which appear in $f$. D9 and D10 also use paired variables. In D9, the mean depends quadratically on one member of the pair. The second member is a quadratically-transformed, but corrupted version, of the relevant feature. For D10, one member of the pair influences the mean and the other its variance. For this dataset, we only consider the Gaussian response case.

\begin{table*}
\caption{Synthetic dataset construction used in Section \ref{sec:synthetic_examples}.
Latent variables and noise terms are independent
standard Gaussian variables and the response column gives the mean function $f(x)$
unless otherwise indicated. The structured features are denoted by $\mathcal{S}$; in our instantiation this is simply $\{1, \dots, 6\}$. In D4 the groups $G_{k}$ are pairs of adjacent variables, and in D7, $q_{c} = 2$ is the number of chains and $L_{k} = 3$ denotes the terminal leaf node. For all datasets, we use $\beta = 4$ and $\gamma = 3$. The assignment of each feature in $\mathcal{S}$ to the different notions of null importance is given in Table \ref{tab:simulation_null_notions}.}
\label{tab:mean_functions}

\renewcommand{\arraystretch}{1.55}
\begin{tabularx}{\linewidth}{@{}l@{\hspace{2.5em}} l@{\hspace{4.5em}} X X@{}}\toprule
\textbf{ID}
&
\textbf{Mechanism}
&
\textbf{Features}
&
\textbf{Response signal}
\\
\midrule

D1
&
Linear
&
$x_j \stackrel{\mathrm{iid}}{\sim} \mathcal{N}(0,1)$
&
$\displaystyle
\beta \sum_{j\in\mathcal S} x_j
$
\\
\midrule
D2
&
Quadratic
&
$x_j \stackrel{\mathrm{iid}}{\sim} \mathcal{N}(0,1)$
&
$\displaystyle
\gamma \sum_{j\in\mathcal S}(x_j^2-1)
$
\\
\midrule
D3
&
Pairwise products
&
$x_j \stackrel{\mathrm{iid}}{\sim} \mathcal{N}(0,1)$
&
$\displaystyle
\gamma \sum_k x_{2k-1}x_{2k}
$
\\
\midrule
D4
&
Grouped parity
&
$x_j \stackrel{\mathrm{iid}}{\sim} U[-1,1]$
&
$\displaystyle
-\gamma \sum_k
\prod_{j\in G_k}\operatorname{sign}(x_j)
$
\\
\midrule
D5
&
Dependent features
&
$\begin{aligned}
x_{2k-1} &= z_k,\\
x_{2k}   &= z_k+\tau\epsilon_k
\end{aligned}$
&
$\displaystyle
\gamma\sum_k x_{2k-1}
$
\\
\midrule
D6
&
Confounding
&
$\begin{gathered}
x_j = z_j+\tau\epsilon_j,\\
z_j\ \text{unobserved}
\end{gathered}$
&
$\begin{aligned}
\eta(z) &= \gamma\sum_{j\in\mathcal S}z_j,\\
f(x)    &= \gamma\sum_{j\in\mathcal S}x_j
\end{aligned}$
\\
\midrule
D7
&
Mediated chains
&
$\begin{aligned}
x_{k,1} &\sim \mathcal{N}(0,1),\\
x_{k,\ell} &= x_{k,\ell-1}
             +\tau\epsilon_{k,\ell}
\end{aligned}$
&
$\displaystyle
\frac{\gamma}{\sqrt{q_c}}
\sum_{k=1}^{q_c}x_{k,L_k}
$
\\
\midrule
D8
&
Redundant pairs
&
$\begin{aligned}
x_{2k-1} &= z_k,\\
x_{2k}   &= z_k+\delta\epsilon_k
\end{aligned}$
&
$\displaystyle
\gamma\sum_k
\frac{x_{2k-1}+x_{2k}}{2}
$
\\
\midrule
D9
&
Quadratic proxy
&
$\begin{aligned}
u_k &\sim U[-1,1],\\
x_{2k-1} &= u_k,\\
x_{2k}   &= u_k^2+\delta\epsilon_k
\end{aligned}$
&
$\displaystyle
\gamma\sum_k
\left(x_{2k-1}^2-\frac13\right)
$
\\
\midrule
D10
&
Heteroscedastic 
&
$\begin{aligned}
x_{2k-1} &= m_k,\\
x_{2k}   &= v_k
\end{aligned}$
&
$\begin{aligned}
\mathbb E[Y\mid x]
    &= \gamma\sum_k m_k,\\
\operatorname{sd}(Y\mid x)
    &= \sigma_y
       \exp\left\{
       \frac{\kappa}{2\sqrt q}
       \sum_k v_k
       \right\}
\end{aligned}$
\\

\bottomrule
\end{tabularx}
\end{table*}

To ensure our measured importances reflect the data generating mechanism rather than a model-based approximation to it, we formed predictions using the ground truth function $f$. This allows us to compare null importance notions induced by the mechanism itself, bypassing error from estimating $f$. For D1 -- D3, the structured features $\mathcal{S}$ are nonnull under every notion of null importance. They are useful for comparing methods but do not distinguish notions of null importance. The remaining mechanisms do. For example, D9 is based on Example \ref{ex:failure_bayes} and distinguishes conditional from risk nulls. The true response depends on a quadratic transformations of $x_{1}, x_{3}$, and $x_{5}$, but each of these true features is accompanied by an already transformed proxy. These proxies are conditionally null because they contribute no new information over the original features, but since $\mathcal{F}$ is chosen to be the set of linear functions, they are risk-nonnull. Similarly, D8, based off Example \ref{ex:conditional_not_functional}, provides features that are functionally and marginally nonnull, but conditionally and risk null. Any one member of the redundant pair can be removed without losing information.

For each dataset, we applied representative feature attribution (Sage\footnote{A risk-based Shapley method \cite{CovertLundbergLee2020SAGE}.}, integrated gradients), selection (correlation, lasso, knockoffs, LOCO, GCM, MDI), and functional sensitivity methods (PDP variance, permutation importance). For MDI alone, we did not use the ground truth $f$ and instead fit a CART model, since MDI is not well-defined without a tree model. For the risk-based methods, we estimate a risk $R^{\ast} = \inf_{f \in \mathcal{F}} R\left(f\right)$ for an appropriate $R$ and $\mathcal{F}$. We use a 70/30 train-test split, with squared error loss and cross-entropy on the test set for continuous and binary $y$, respectively. All risk-based methods use the same split, and Sage and MinSHAP share coalitions. For all datasets except D9, we used gradient boosting ensembles to represent $\mathcal{F}$. For D9 we deliberately restricted $\mathcal{F}$ to linear predictors, illustration a violation of Bayes completeness.

Figure \ref{fig:relative_importances} summarizes the resulting scores for the regression case when $n = 500$. The analogous figures for $n = 5000$ and classification re given in Supplementary Figures \ref{fig:null_mass_5000} - \ref{fig:null_mass_classification_5000}. Color encodes
\begin{align*}
\frac{\sum_{k \in \mathcal{N}}\left|\varphi_{k}\right|}{\sum_{j}\left|\varphi_{j}\right|}
\end{align*}
where $\mathcal{N}$ denotes the null set defining each panel. The importance scores are held fixed across columns, only the null notion $\mathcal{N}$ varies (Table \ref{tab:simulation_null_notions}).

\begin{table*}[!t]
\centering
\caption{Notions of null associated with each feature from datasets D1--D10 in Table \ref{tab:mean_functions}. In addition to these structured features, each dataset includes twelve independent standard normal features that are null under every notion.}
\label{tab:simulation_null_notions}
%\begin{tabular}{llll}
\begin{tabularx}{\linewidth}{
  >{\hsize=0.3\hsize}X % id
  >{\hsize=1\hsize}X % subset
  >{\hsize=1.1\hsize}X % null under
  >{\hsize=1.1\hsize}X % nonnull under
}
\toprule
 & Feature subset & Null under notion(s) & Non-null under notion(s) \\
\midrule
D1--D3 & all              & ---                              & all notions \\
D4     & all & marginal                         & conditional, risk, functional, causal \\
D5     & proxies $x_{2k}$     & conditional, risk, functional, causal & marginal \\
D6     & all & causal                           & marginal, conditional, risk, functional \\
D7     & internal nodes $x_{kl}$ s.t. $ l< L_{k}$ & conditional, risk, functional  & marginal, causal \\
D8     & Copies $x_{2k}$      & conditional, risk, causal        & marginal, functional \\
D9     & Transformations $x_{2k}$ & conditional, functional, causal  & marginal, risk \\
D10    & Variance features $x_{2k}$  & risk, functional            & marginal, conditional, causal \\
\bottomrule
\end{tabularx}
\end{table*}

\textbf{Correlation}. Since correlation is sensitive to linear association, it fails to assign importance to nonnull features in D2--D4, which are nonlinear. In D5, it places importance on features that are marginally nonnull but null under every other notion. In D7, it places high importance on members of the mediated chain, which are nonnull only from marginal and interventional causal perspectives.

\textbf{Lasso}. Like correlation, Lasso underperforms on nonlinear datasets. In D5, it does not assign importance to features $x_{2}, x_{4}$, and $x_{6}$, since it selects only one member from correlated groups. Therefore it outperforms risk-based methods across every notion on this dataset. A similar pattern appears in D8, where the correlated proxies are given coefficients equal to zero. On D9, the Lasso's failure to detect the nonlinear association leads it to assign importance to the pretransformed features, which are null according to conditional and risk notions.

\textbf{Knockoffs}. Knockoffs performs similarly to Lasso, though it is more conservative, reflected in the smaller null importance fractions in D2 - D4. On D8, it assigns importance to some correlated proxies, though to a lesser extent than functionally-sensitive methods. A closer comparison with correlation in D5 is given in Figure \ref{fig:minshap_details}a. Knockoffs send conditionally null variables to zero, despite their high correlation with the response. The ability to distinguish conditionally important variables improves with sample size $n$ because conditional null importance is harder to evaluate than marginal null importance.

\textbf{LOCO}. Since LOCO measures the drop in risk after fitting a flexible model, it detects the nonnull quadratic features in D2. On D3 and D4, where the signal is based entirely on interactions, LOCO can in principle detect the risk drop. However, this depends on accurately estimating the interaction structure, which it fails to do here. On the redundant pairs D8, it fails even under marginal and functional null notions. It assigns little importance to any feature in $\mathcal{S}$, since dropping one feature doesn't substantially reduce performance. Estimation error means that features $\mathcal{S}^{C}$ have scores on a comparable scale, making the signal invisible.

\textbf{GCM}. GCM performs similarly to LOCO except on D2. Since the signal features there are independent, $\mathbb{E}\left[X_{j} \mid X_{-j}\right] = 0$. After residualization, GCM tests the association between $X_{j}$ and $X_{j}^{2} - 1$, but since it uses covariance, it cannot detect the quadratic dependence. 

\textbf{PDP variance, Integrated Gradients, and Permutation Importance}. These methods flag the same variables as nonnull across all datasets. Their sensitivity to nonlinear features allows them to perform well on D2--D4 and D9. In D8 they place high importance on the correlated proxies, which are considered null under conditional, causal, and risk null notions, but which are functionally related to the response.

\textbf{Sage}. Sage performs worse than all other methods on the mediated chains D7. There are coalitions where all terms downstream of an internal node are absent. On these coalitions performance improves when adding the internal node. Averaging over coalitions, the internal nodes are flagged as important, reinforcing the discussion of Section \ref{sec:feature_attribution_methods}. See also Figure \ref{fig:minshap_details}b.

\textbf{MinSHAP}. Though it is formed with the same coalitions as Sage, MinSHAP does not place importance on any intermediate nodes in D7. For D3 and D4, MinSHAP assigns zero importance to all features due to failures of the null monotonicity condition. Features on their own have no predictive value but when their interaction partners are present they become important. MinSHAP assigns zero importance to all features in D5 as well, but for a different reason. For the features $x_{2k - 1}$, the population risk is nonzero but small, due to the presence of the highly correlated proxy. On finite samples, the risk differences are subject to Monte Carlo error, and taking the minimum over coalitions biases the importance downwards to zero.

\textbf{MDI}. MDI performs similarly to LOCO except on D8, where the functional method places high importance on the correlated proxies. This faithfulness to the functional form of the data-generating mechanism makes it incompatible with conditional, causal, or risk-based null notions.

The purpose of this example is not to rank one method over others, but rather to demonstrate the concrete differences in data characteristics to which these methods are sensitive. Even with relatively simple generative mechanisms, differences in functional form and dependence structure lead to substantive differences in which features are considered important.

\begin{figure}
    \centering
    \includegraphics[width=1.0\linewidth]{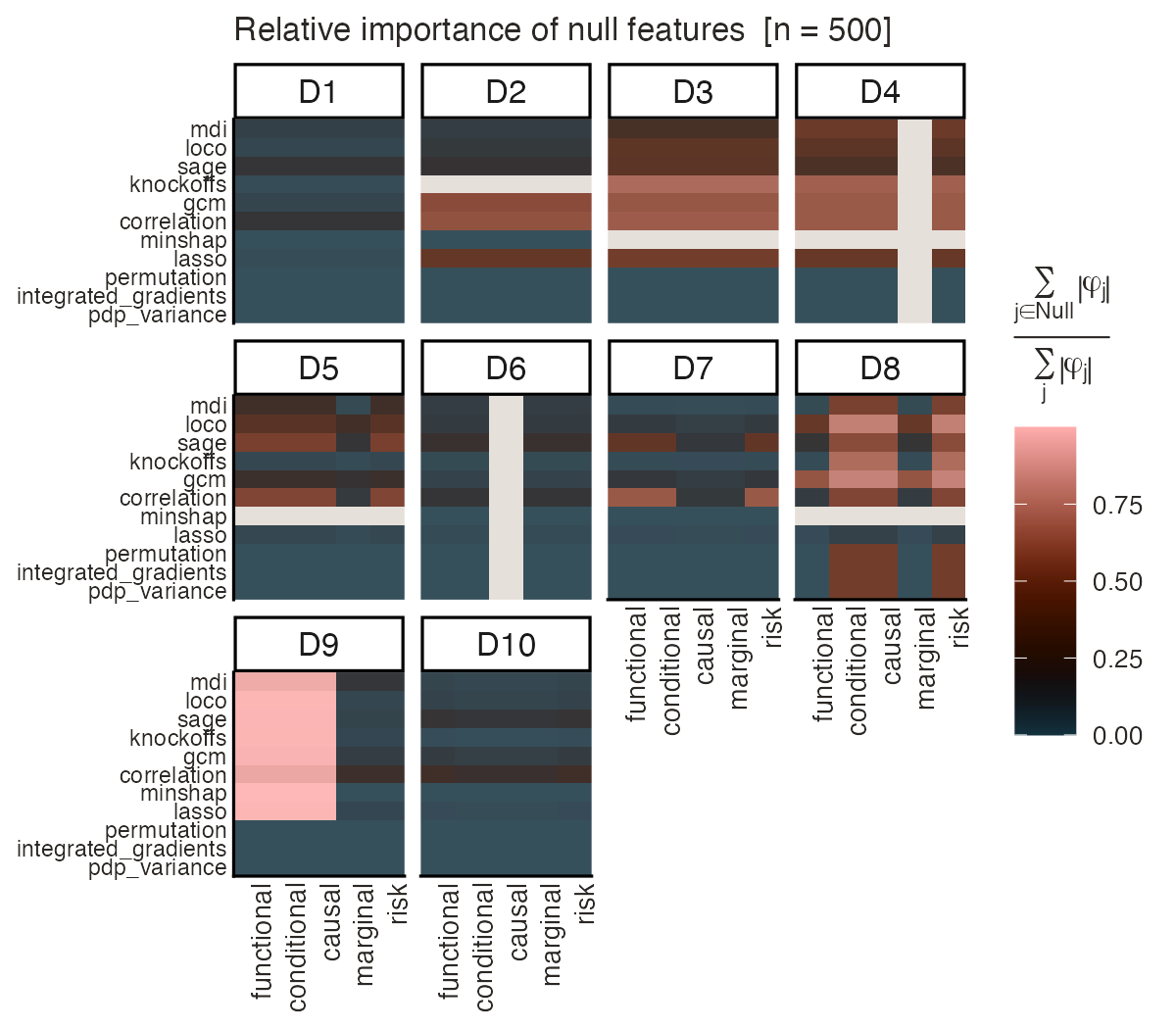}
    \caption{Relative variable importance from null or signal features when they are assigned to alternative notions of null importance. Each row is a synthetic data set as described in Table \ref{tab:mean_functions}, averaging importances across all seeds among datasets with $n = 500$ and the regression task. Each column gives a variable importance method. Gray cells indicate cases where the relative variable importance is undefined, either because all features are considered nonnull (D4 marginal and D6 causal) or because a method did not assign positive importance to any feature.}
    \label{fig:relative_importances}
\end{figure}

\begin{figure}
    \centering
    \includegraphics[width=1.0\linewidth]{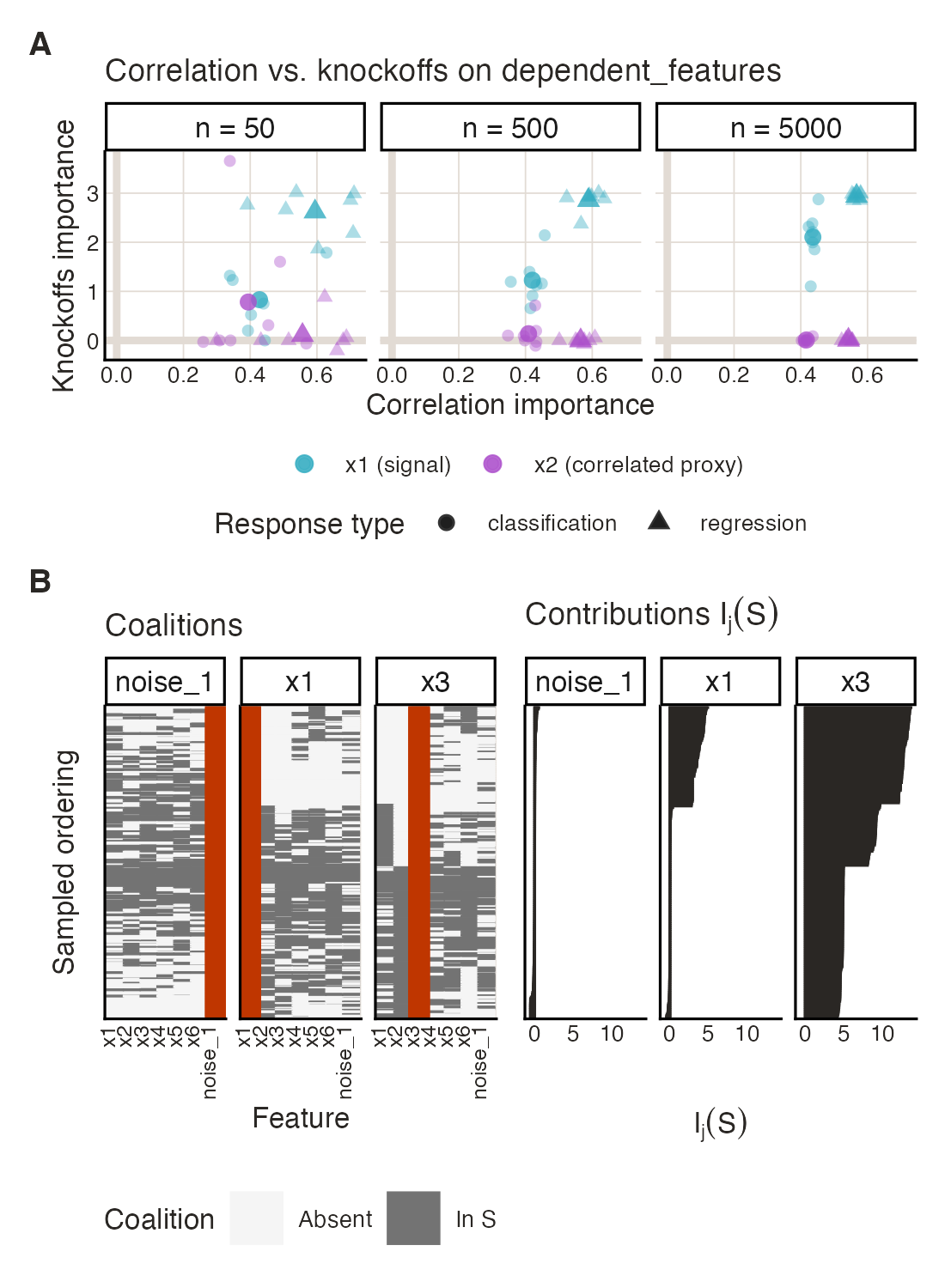}
    \caption{a. A comparison between correlation and knockoff importance scores across sample sizes on dataset D5. Only Knockoffs distinguishes between conditionally relevant and null features. b. The relative contributions of each coalition on the mediated chains task, as used by Sage and MinSHAP. In the left panel, each row along the $y$-axis corresponds to a coalition and shows features $S$ in grey and new addition $j$ in red. The rows of this figure are aligned with those in right panel, where each bar displayed $I_{j}\left(S\right)$.}
    \label{fig:minshap_details}
\end{figure}

\subsection{Local Attributions for functional local null on MNIST}

We can similarly evaluate local attribution methods from a null importance perspective. In particular the targeted null here is the local functional null $\mathcal N_j^{\mathrm{func-local}}$. We apply representative local attribution methods to the MNIST digit classification task. Our model is a pretrained two-stage ResNet with hidden layer widths of 32 and 64 (\texttt{fxmarty/resnet-tiny-mnist} on HuggingFace). We explain 10 examples of the digit ``0.'' The examples were chosen as the five examples with the highest and lowest predicted probability of the correct class, respectively. For local attribution, we apply three gradient-based methods (saliency maps, gradient $\times$ input, and integrated gradients) and two Shapley values-based methods (SHAP and MinSHAP). Gradient $\times$ Input is defined as $$\frac{\partial f(x)}{\partial x_j}\cdot(x_j-x'_j)$$ where $x_j'$ is the same baseline used for integrated gradients. Saliency maps and Gradient $\times$ Input use the gradient evaluated at the observed input, while integrated gradients averages gradients along the path from the baseline to the observed input (see Equation \ref{eq:integrated_gradients}). For both gradient based methods, we used a discretization size of 0.001, and for integrated gradients we considered ten steps. For SHAP and MinSHAP, we considered the formulation with value function derived from instance-level predictions $f\left(x\right)$ rather than the risk-based formulation, which only applies to global attribution. For SHAP, we used a KernelSHAP approximation with 100 background samples and 7300 coalitions. For MinSHAP, we used an IME approximation with 40 coalition orderings. In both cases, we use the value function $V(S)=\mathbb{E}\left[f(X) \mid \operatorname{do}\left(X_S=x_S\right)\right]$, which is appropriate to local attribution. The coalition budgets were chosen to give SHAP and MinSHAP roughly comparable runtimes (1 hour under current settings), allowing the comparison to reflect differences between the attribution methods rather than differences driven by unequal computational runtimes. For each panel, we centered and scaled attributions to mean zero and variance one, ensuring that the attributions are on a comparable scale across panels. Attribution $z$-scores were clipped at $\pm 4$ to prevent outliers from dominating the heatmap.

The resulting attributions are shown in Figure \ref{fig:mnist_attributions}. For correctly predicted samples, the saliency maps are close to zero due to vanishing gradients. Since these examples were those with the highest predicted probabilities of the correct class, their final layer activations are saturated and the gradients are nearly zero. When the saliency gradient is nonzero, Gradient $\times$ Input gives clearer attribution patterns by downweighting pixels that are close to the zero baseline. However, because it still uses the gradient evaluated at the observed input, it also becomes zero when the saliency gradient vanishes. This difficulty is bypassed by integrated gradients, which produces gradients with large magnitudes at intermediate values along the path integral, even for correctly classified examples. This comparison suggests that multiplying by the input can remove much of the background variation seen in saliency, while integrating gradients along the path is what allows integrated gradients to avoid the saturation problem. For the Shapley value-based methods, SHAP remains relatively diffuse, while MinSHAP produces a clearer attribution patterns around the digit. The attribution patterns could become clearer for both methods if more coalitions are included in computation.

This case study suggests that although these local attribution methods all returned images with pixel-level importances, they can differ because they target different nulls, like the functional vs. causal nulls targeted by the gradient and SHAP-based methods, respectively. Further, even when they target the same null, differences in estimation strategy can highlight or suppress particular pixels, like in the case of saliency vs. Gradient $\times$ Input explanations.

\begin{figure*}
    \centering
    \includegraphics[width=0.8\linewidth]{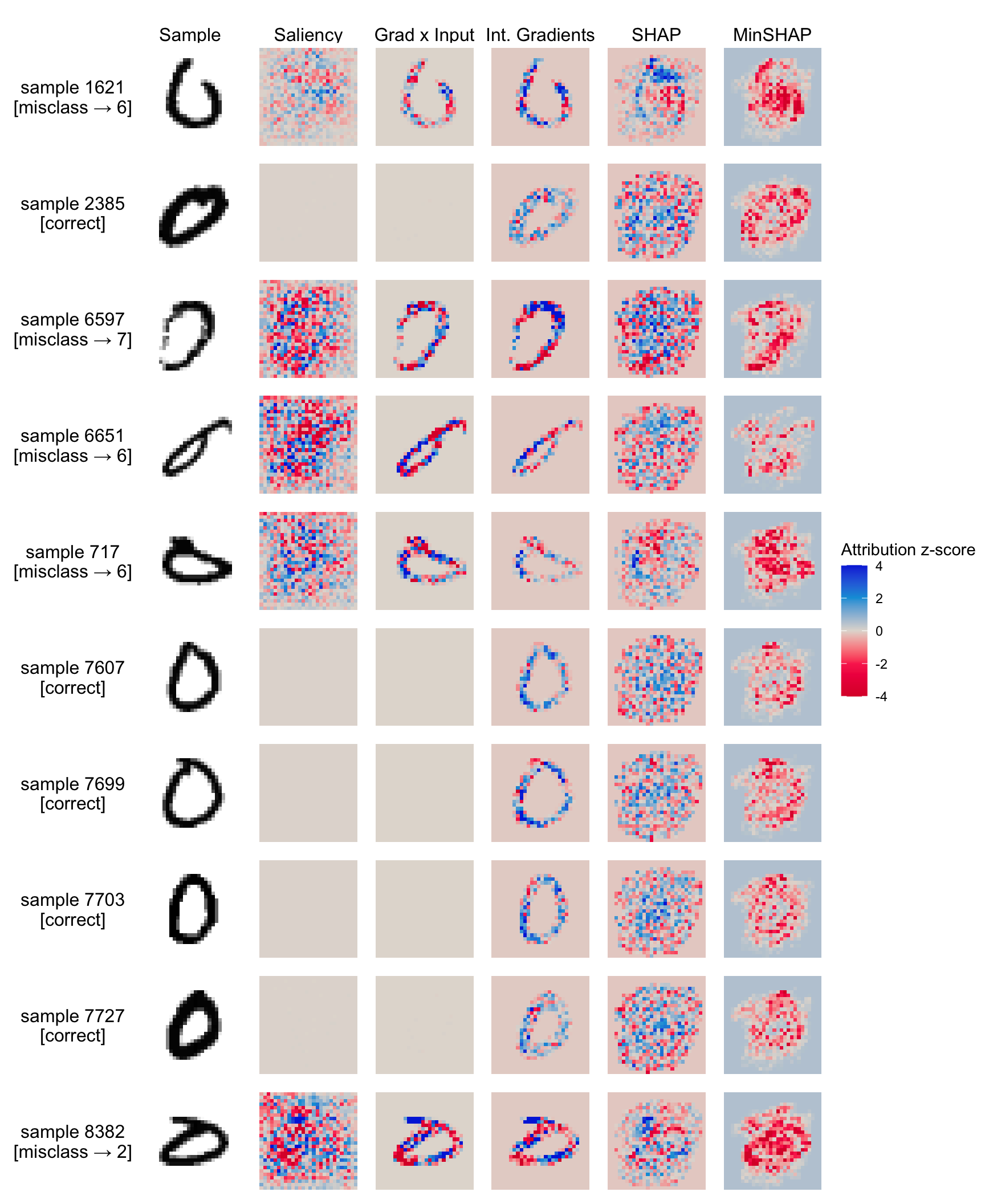}
    \caption{Local feature attributions for selected ten MNIST digit-0 examples. Each row corresponds to one input image, with its prediction status shown on the left, and each column gives the attribution from a different explanation method. Attribution scores are standardized within each panel and displayed as $z$-scores. Blank saliency and Gradient $\times$ input panels reflect near-zero gradients.
}
    \label{fig:mnist_attributions}
\end{figure*}

\subsection{Variable Importance in TCGA Multiomics}

The synthetic data analysis showed that alternative definitions of null importance strongly influence interpretation. While synthetic data provide ground truth, the data generating mechanisms are necessarily simplified. Next we evaluate the impact of alternative null importance notions on a multiomics analysis problem. A common strategy in modern genomics is to study disease mechanisms by predicting phenotypes from high-dimensional sequencing assays. To ensure accurate predictions, it is often necessary to use black-box models to capture nonlinearities and interactions, but the scientific goal is mechanistic understanding, not simply to classify new samples. A feature importance analysis can flag molecular features to perturb in a follow-up experiment, to see their impact on in-vitro disease models. In practice many methods have been used. How do their different sensitivities to null importance notions affect the resulting scientific conclusions?

We analyzed 461 samples from The Cancer Genome Atlas (TCGA). Each sample includes features from mRNA, miRNA, and RPPA (protein) molecular assays. The data include 362 ductal and 99 lobular tumors. Before filtering, the mRNA, miRNA, and RPPA assays contain 59427, 2238, and 131 features respectively. We preprocessed the data following \cite{Novoloaca2024}, which benchmarked multiomic prediction models on this dataset. Because several explanation methods scale poorly in the number of features, we applied more stringent filtering than \cite{Novoloaca2024}. We kept the 150 most variable mRNA features, 100 most variable miRNA features, and 50 most variable RPPA features (a total of 300 features), while \cite{Novoloaca2024} retained the 1000 most variable in total after combining assays. Otherwise, we followed their preprocessing procedure, first imputing missing values with their mean and then residualizing each feature on age, gender, and race. We trained random forest models to predict cancer type. Each model included 50 trees. To evaluate out-of-sample performance, we used stratified five-fold cross-validation. Within each training fold, we evaluate the hyperparameters $\texttt{ min\_samples\_leaf } \in\{1,3,6\}, \quad \texttt{ max\_features } \in\left\{\log _{10} d, \sqrt{d}, \frac{d}{5}, \frac{d}{2}\right\}$ and selected the combination that maximized the out-of-bag Matthews correlation coefficient (MCC). The resulting model was evaluated using that fold's held-out samples, and pooling across folds yields an out-of-samples MCC of 0.618 and area under the curve of 0.770. The classifier used for methods below was fit separately on the full dataset with hyperparameters of $(\texttt{min\_samples\_leaf}, \texttt{max\_features}) = (6, 150)$ again chosen using out-of-bag MCC.

We have applied twelve methods to interpret this final model: MDI, TreeSHAP, permutation importance, PDP variance, integrated gradients (discretized to 64 steps), correlation, lasso coefficients on standardized inputs, Gaussian model-X knockoff statistics, negative log $p$-value from a GCM conditional independence test, LOCO, Sage, and MinSHAP. For the risk-based methods, we used an XGBoost classifier fit with a 70/30 train/test split, estimating risk using binary cross entropy,
\begin{align*}
\hat{R}(S) = -\frac{1}{N_{\text{test}}} \sum_{i \in \text{Test}} &\Bigg[Y_i \log \hat{g}_{S}(X_{i, S}) \\&+ (1 - Y_i) \log\left(1 - \hat{g}_{S}(X_{i, S})\right)\Bigg]
\end{align*}

where $\hat{g}_S(x)$ estimates $P(Y = 1 \mid X_S = x)$. Each method $m$ provides $J = 300$ importance estimates,
$$
\varphi^{m} = \left(\varphi_1^{m}, \ldots, \varphi_{J}^{m}\right).
$$
To account for randomness in the methods, we recomputed explanations using four seeds and averaged the resulting estimates to obtain $\bar{\varphi}_{j}^{m}$. To put importance measures on the same scale, we normalize using,
$$
\pi_{j}^{m} = \frac{\left|\bar{\varphi}_{j}^{m}\right|}{\sum_{k = 1}^{300}\left|\bar{\varphi}_{k}^{m}\right|}, \qquad j = 1, \dots, J.
$$
We applied PCA to the matrix $\Pi = \left(\pi_{j}^{m}\right) \in \mathbb{R}^{M \times J}$. The first two principal components are displayed in Figure \ref{fig:tcga_pca}. We excluded MinSHAP from the PCA because it assigned zero importance to every feature. This is consistent with the high feature-feature correlations present in these data -- across features, the maximum absolute correlation with any other feature has a median of 0.72.

The first principal component (PC1) explains 80.1\% of the variation in normalized feature importance profiles. This dimension contrasts methods with diffuse vs. concentrated feature profiles. Methods with larger PC1 scores, including correlation and the risk-based methods, spread their importances across many features (Supplementary Figure \ref{fig:cumulative_importance}). If the scientific question is whether a feature contributes predictive information or is marginally related to the response, the diffuse profiles for these methods suggests that most are. In contrast, methods with smaller PC1 scores concentrate their normalized importance on fewer features. These features are also separated along PC2 (13.8\% of variance explained) and appears to contrast methods emphasizing conditional independence vs. those sensitive to functional importance. Figure \ref{fig:tcga_examples} examines this interpretation using the partial dependence profiles for the three features with largest absolute PC2 loadings. Integrated gradients assigns high importance to features like Beta-catenin, where the predicted class probabilities rapidly transition between zero and one, even though Knockoffs did not detect conditional dependence between the feature and the tumor type given the remaining features. In contrast, MIMAT0000267, which has the largest PC2 loading, has a relatively flat partial dependence profile but receives high importance under both the Lasso and Knockoffs, both of which account for the remaining features. E-cadherin also has a large positive loading on PC2, but provides an intermediate case, receiving modest importance under both functional and conditional relevance notions. Because correlation, GCM, and LOCO spread importance across many features, their normalized importance is spread thin, leading to shorter feature-level bars. The bar lengths should not be compared directly with more concentrated methods, as they are not absolute importance magnitudes.

Features identified under any of these null notions could motivate experimental follow-up. We emphasize only that if the method's target does not match the scientific question -- for example, if we were searching for rapid changes in probability (functional relevance) using Knockoffs (targeting conditional statistical relevance), or uniquely related features (conditional statistical relevance) using integrated gradients (targeting functional relevance) -- then we may miss a pattern relevant to the motivating question.
\begin{figure*}
    \centering
    \includegraphics[width=0.75\linewidth]{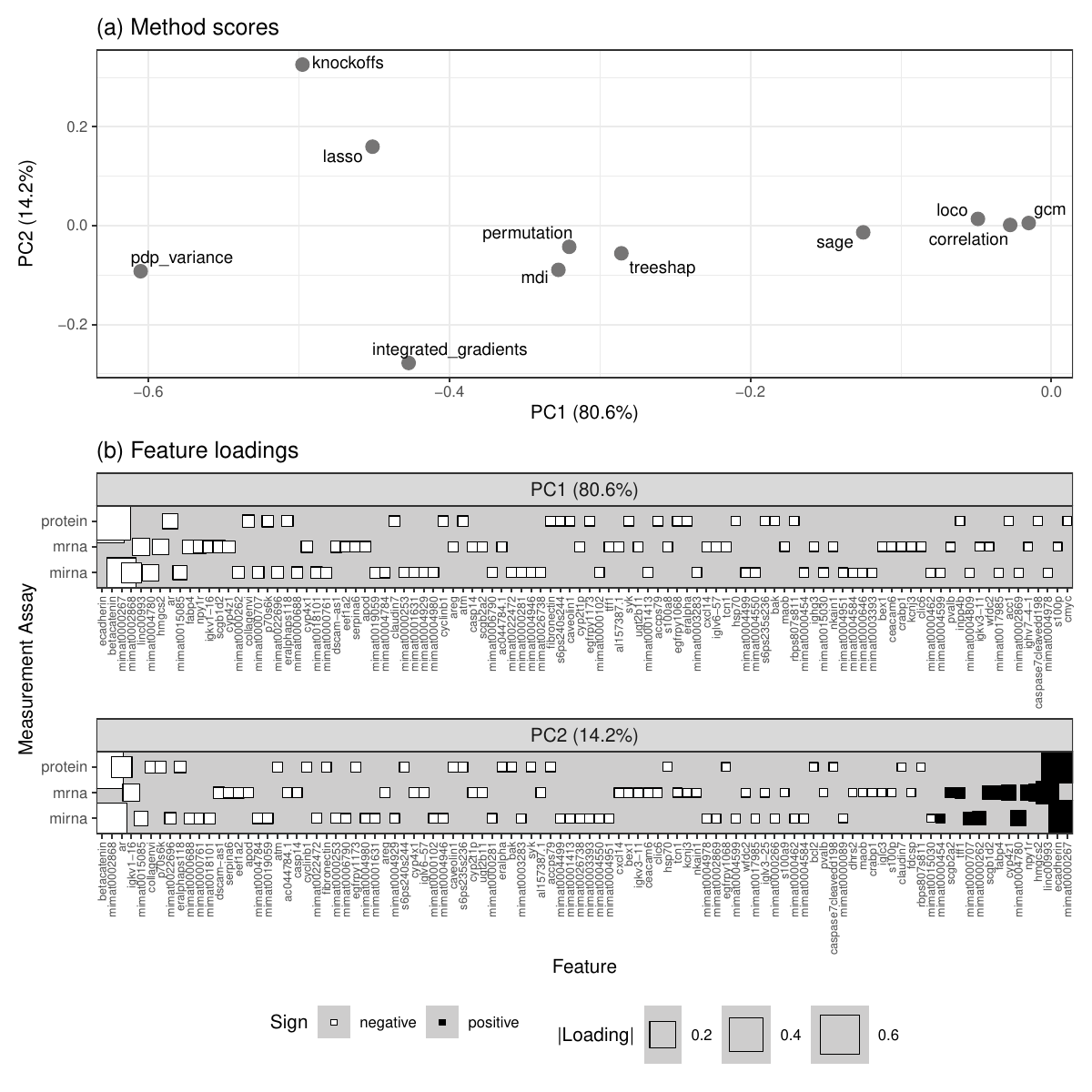}
    \caption{PCA of $\pi_{j}^{m}$ importances from the TCGA example. Panel (a) organizes the methods according to the top two principal component scores. Panel (b) is a Hinton diagram that shows the 50 features that contribute the largest loadings to either PC1 or PC2.}
    \label{fig:tcga_pca}
\end{figure*}

\begin{figure}
    \centering
    \includegraphics[width=1.0\linewidth]{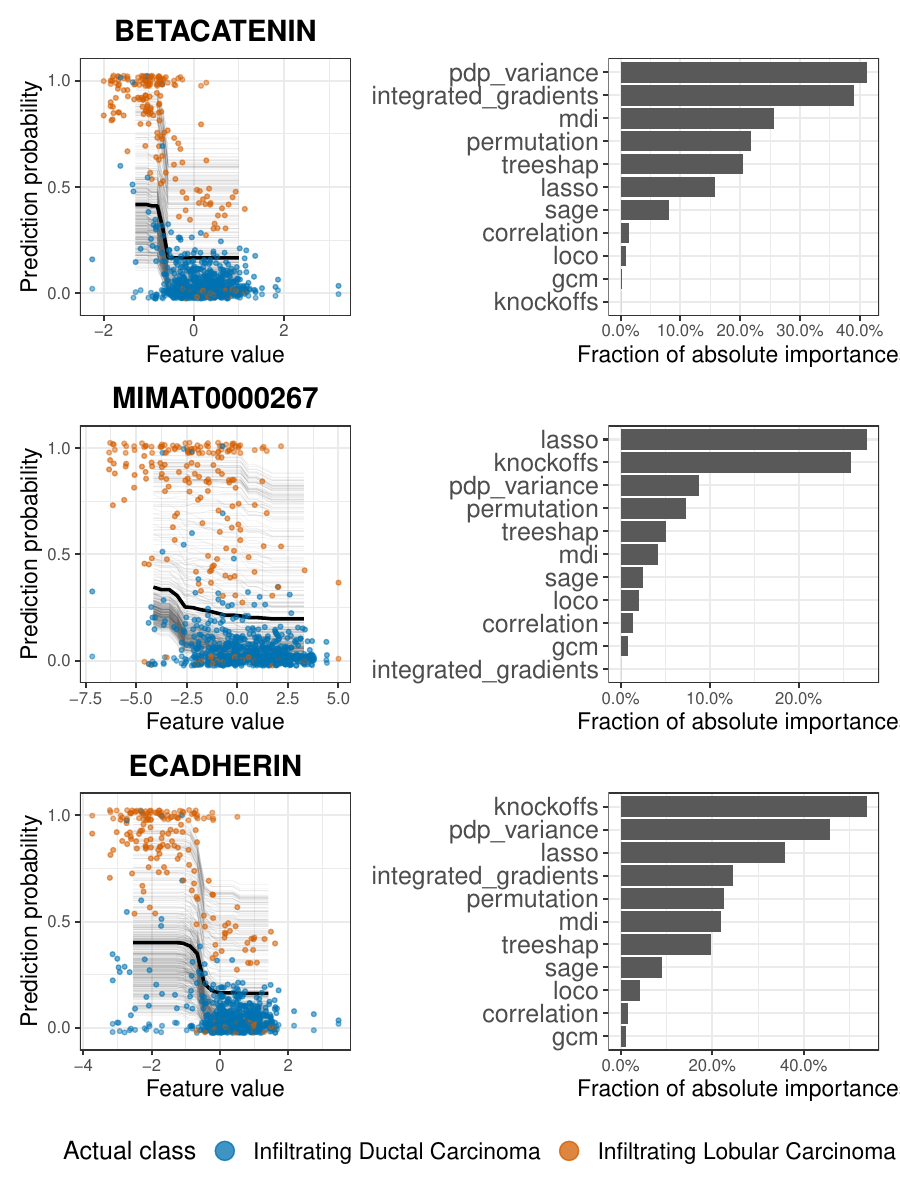}
    \caption{Features with large-magnitude loadings on PC2 in the TCGA multiomics analysis. Left panels show partial dependence profiles for each feature, with points colored by observed cancer type. Right panels show the fraction of absolute importance assigned to the corresponding feature by each method.}
    \label{fig:tcga_examples}
\end{figure}

\section{Discussion}

A central message of this paper is that whether a feature is ``important'' depends on the question being asked, the data-generating process that produced the observations, and the algorithm or value function used to analyze them. These three components are often considered separately: domain scientists formulate questions about which variables matter, data analysts work with an observed joint distribution, and learning algorithms quantify the contribution of variables to a particular predictive or functional objective. The notion of \emph{null importance} provides a common language for connecting these perspectives.

The key advantage of this language is that it makes explicit what it means for a feature to be irrelevant. In the framework developed in this paper, functional, statistical (marginal and conditional), risk-based, and causal notions of null importance correspond to different answers to the question of what should remain invariant when a feature is removed or changed. These notions are not generally equivalent. Their relationships depend on properties of the data-generating process, such as dependence, redundancy, noise, confounding, and interactions, as well as on assumptions about the model class and the value function used by the algorithm.

This separation is important because the scientific question and the available data need not identify the same notion of relevance. A domain scientist may ask whether a variable is causally relevant to an outcome, while the available data support only an observational or predictive question. Similarly, an algorithm may assign a feature high importance because it improves prediction, even when the scientific question concerns functional necessity. Conversely, a feature may be scientifically relevant but receive little importance under a particular predictive objective because its information is redundant with other variables. Null importance makes these distinctions explicit: rather than asking whether a feature is important in the abstract, one can ask under which notion of irrelevance the feature is null, and under what assumptions that notion corresponds to the scientific question of interest.

The equivalence results in this paper identify conditions under which these different perspectives can be aligned. For instance, Lemma~\ref{LemFuncStat} shows that functional and conditional statistical nullity coincide only under deterministic, non-redundant, and sufficient regimes and Lemma~\ref{LemRiskNull} establishes that risk-based and conditional statistical nullity coincide only under well-specification and sufficiently rich function classes. These results specify when an algorithmic notion of null importance can be interpreted as evidence for a particular scientific notion of irrelevance.

The counterexamples illustrate why such alignment cannot be assumed in general. Redundancy can make a feature statistically informative without making it functionally necessary, interactions can make relevance inherently joint and therefore invisible to certain local measures, and confounding can produce observational relevance without causal relevance. These are not merely failure modes of particular feature importance methods. They are consequences of a mismatch between the scientific question, the structure of the data-generating process, and the value function encoded by the algorithm. In this sense, disagreement between feature importance methods can be scientifically informative because it may reveal that the methods are answering different questions rather than that one of them is simply incorrect. These equivalences and counterexamples were reflected in the synthetic and real data results across different interpretable feature analysis methods.

This perspective also suggests a practical role for null importance as a \emph{communication and validation tool}. In applied work, interpretable feature analysis is often used in collaboration with practitioners, domain scientists, or other stakeholders who have a substantive notion of what it means for a variable to matter. The term ``important'' can conceal substantial ambiguity in such collaborations. A practitioner may mean that changing the variable would change the outcome, a scientist may mean that the variable has a causal effect, and a machine-learning practitioner may mean that removing the variable increases prediction error. By explicitly stating the corresponding null importance, these interpretations can be distinguished and compared. This provides a common vocabulary for discussing what a feature importance analysis establishes, what assumptions connect the analysis to the scientific question, and what conclusions can legitimately be drawn from the result.

Viewed this way, interpretable feature analysis is not only a problem of constructing or estimating importance scores. It is a problem of aligning a \emph{scientific question}, a \emph{data-generating process}, and an \emph{algorithmic value function}. Null importance provides a useful organizing principle for this alignment because it forces the analyst to specify what notion of irrelevance is being tested. The framework developed here consequently provides a way to make assumptions visible, identify when different notions of relevance coincide, and diagnose when they necessarily diverge.

\paragraph{Extensions to learned representations}

Although the paper focuses on raw input features $(X_1,\ldots,X_p)$, the same framework applies to learned or hidden representations. Let $\mathcal U=(u_1,\ldots,u_m)$ denote a collection of explanatory units, where a unit may be a latent feature, learned concept, attention head, or computational component. Null importance can be defined for these units by replacing the raw feature set with $\mathcal U$ and specifying the relevant value function.

The distinction between the scientific question, the data-generating process, and the algorithm is particularly important in this setting. Learned units are themselves constructed by an algorithm and may be redundant, distributed, or non-identifiable. Consequently, a unit that is null under one value function need not be null under another, and algorithmic necessity need not correspond directly to a scientifically meaningful notion of relevance. This issue arises naturally in representation learning and mechanistic interpretability, including work on sparse and monosemantic representations~\cite{BrickenEtAl2023TowardsMonosemanticity,CunninghamEwartGhahramani2023SparseAutoencoder,TempletonEtAl2024ScalingMonosemanticity}. Extending the taxonomy of null importance to learned representations provides a way to study how their components relate to the underlying data-generating process and to the scientific questions that the representation were built to answer.

\paragraph{Extensions from individual features to subsets}

A second extension concerns subsets of features. The analysis in this paper focuses on the null importance of an individual feature $X_j$, but the same questions arise for groups of variables. For a subset $S\subseteq\{1,\ldots,p\}$, one may ask whether $X_S$ is null and, conversely, whether there exists a smallest subset $S$ whose complement is null.

These questions become particularly important when the scientific question concerns a collection of variables, or when the data-generating process contains redundancy and interactions. A feature may be individually null because another feature carries the same information, while a group of features may be jointly necessary. Conversely, several different subsets may provide equivalent information for the value function. Thus, individual feature importance need not identify the smallest scientifically or predictively relevant set of variables. Interaction measures and cooperative-game formulations provide related perspectives on these issues \cite{LundbergLeeShapleyInteractions2018,GrabischRoubens1999Interaction}.

A subset-level extension of null importance would therefore need to specify not only which features are being considered jointly, but also the relevant notion of joint irrelevance and the value function against which it is assessed. As with individual features, this would provide a language for distinguishing scientific relevance from statistical, predictive, or functional necessity at the group level.

\bibliographystyle{imsart-number}
\bibliography{null_importance_literature}
\newpage

\begin{appendix}

\section{Supplementary Figures}
\begin{figure}[H]
\includegraphics[width=1.0\linewidth]{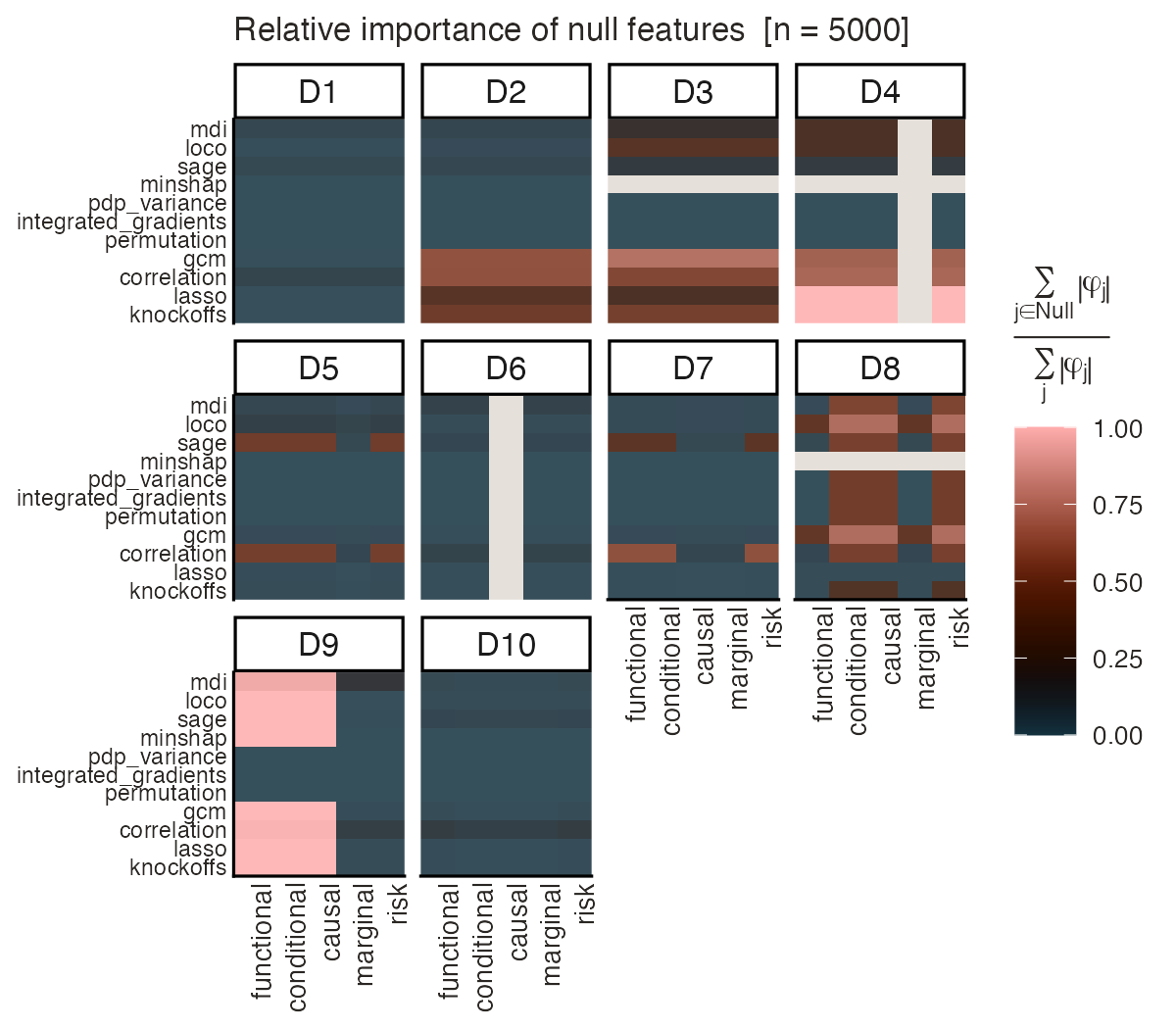}
\caption{The analog of Figure \ref{fig:relative_importances} when $n = 5000$. The same trends appear.}
\label{fig:null_mass_5000}
\end{figure}

\begin{figure}[H]
\includegraphics[width=1.0\linewidth]{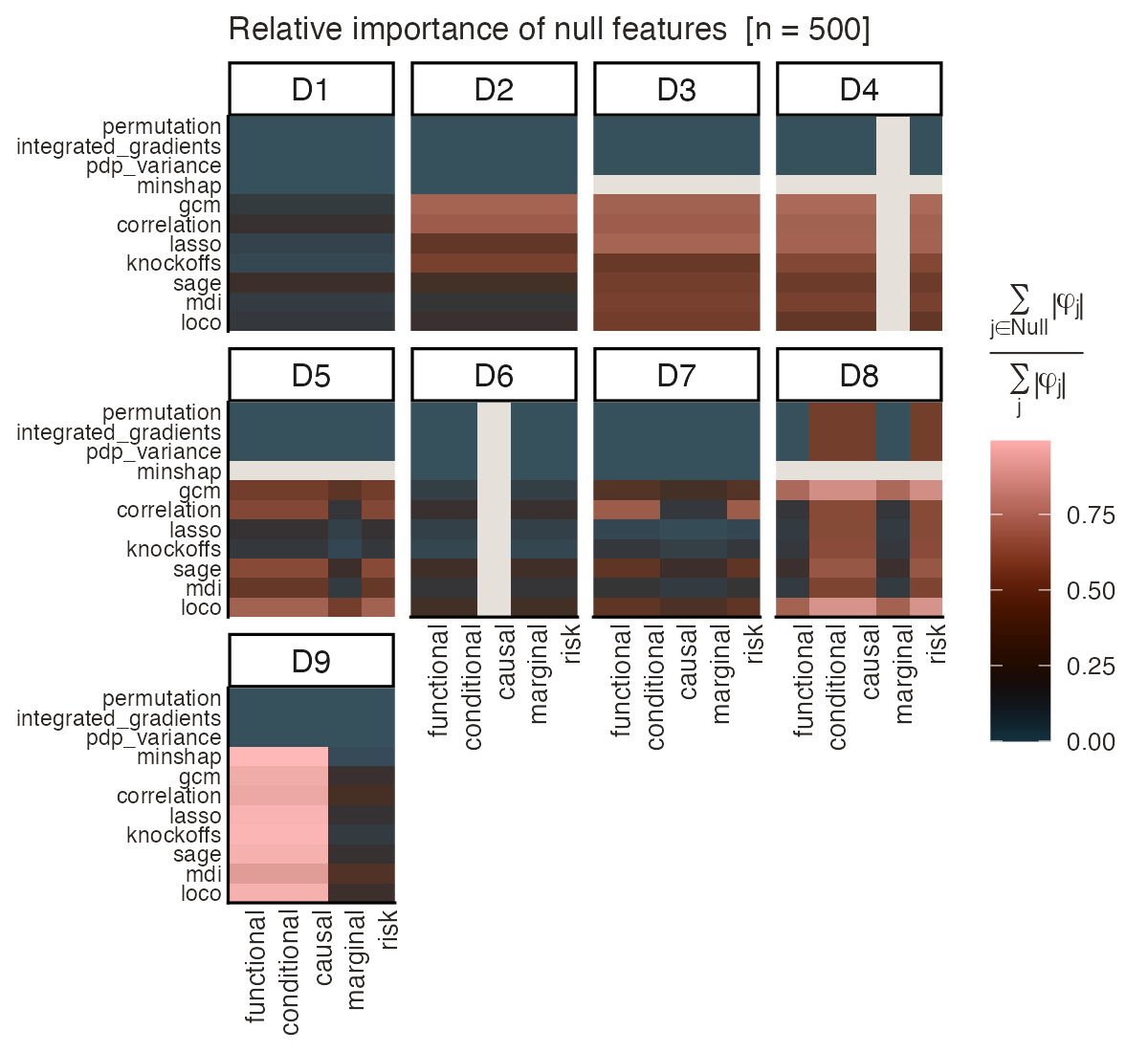}
\caption{The analog of Figure \ref{fig:relative_importances} for classification.}
\label{fig:null_mass_classification_500}
\end{figure}

\begin{figure}[H]
\includegraphics[width=1.\linewidth]{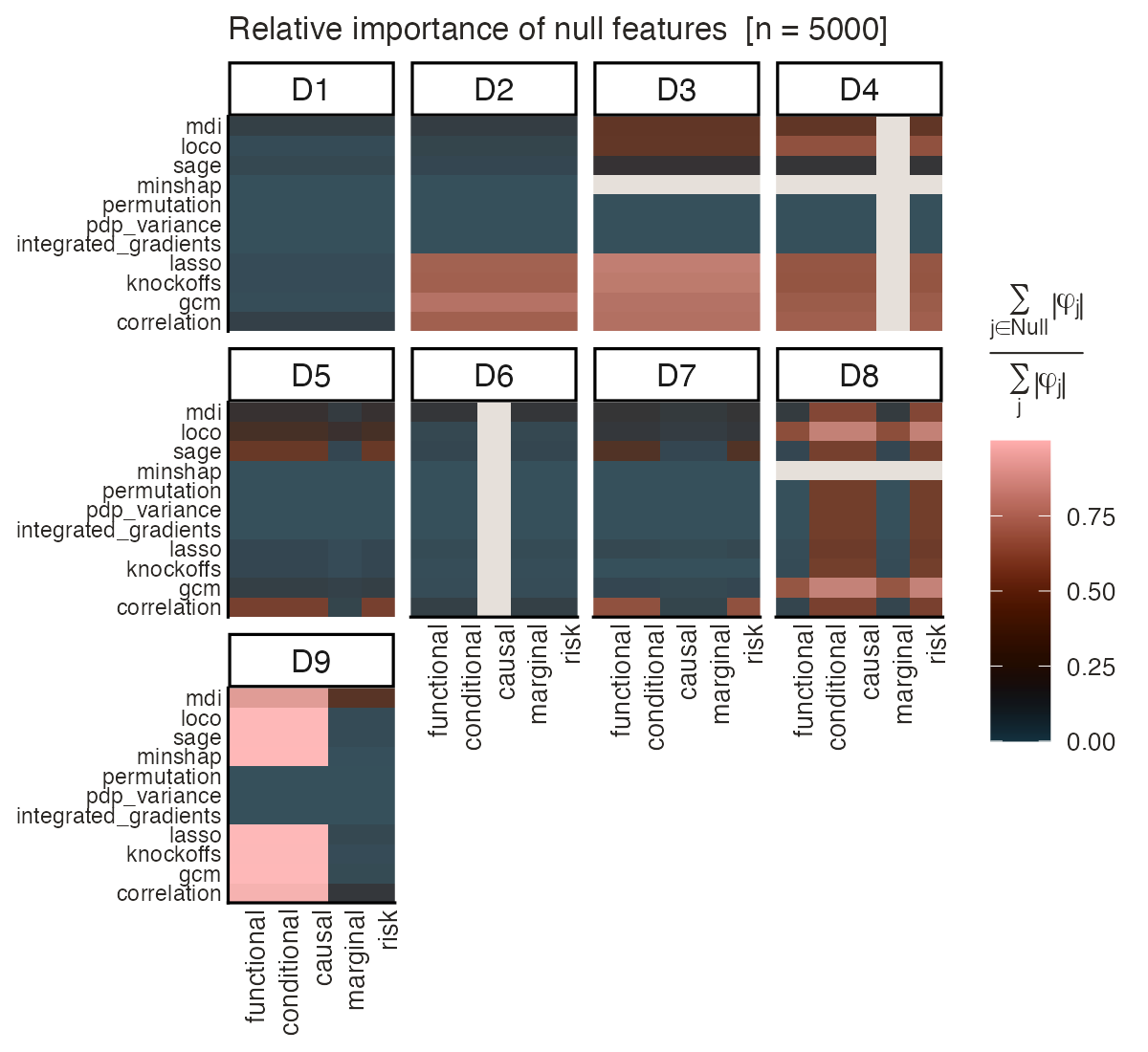}
\caption{The analog of Figure \ref{fig:relative_importances} for classification and when $n = 5000$.}
\label{fig:null_mass_classification_5000}
\end{figure}
\end{appendix}

\begin{figure}[H]
    \includegraphics[width=1.0\linewidth]{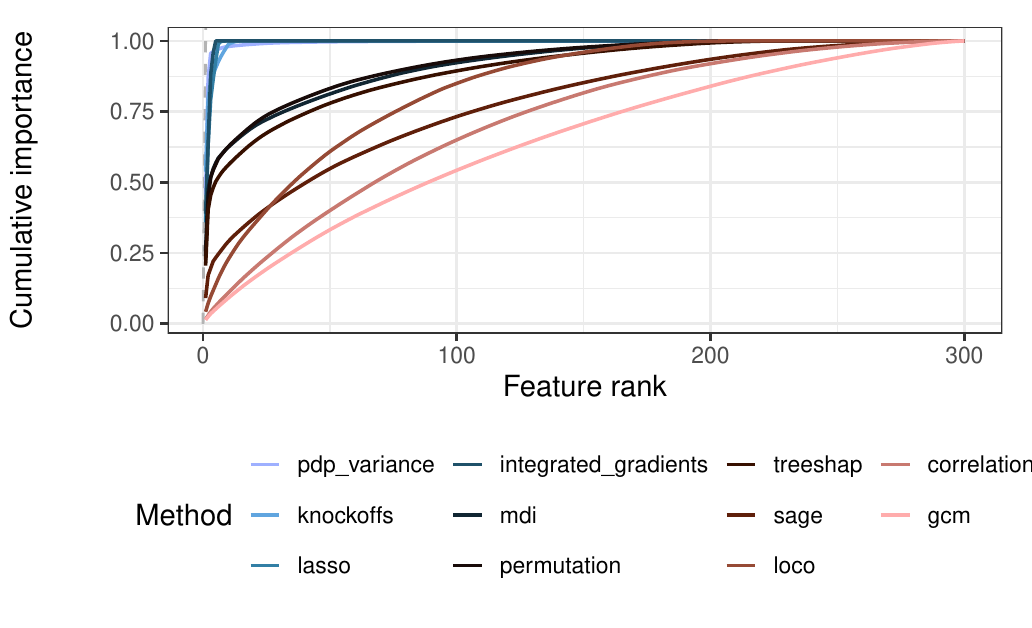}
    \caption{Cumulative variable importances on the TCGA case study. The color gradient has been organized by the PC1 score shown in Figure \ref{fig:tcga_pca}.}
    \label{fig:cumulative_importance}
\end{figure}

\section{Case Study Materials}

Data to replicate the case studies have been deposited on Figshare (\url{https://doi.org/10.6084/m9.figshare.33857824}). A code repository is available on GitHub (\url{https://github.com/krisrs1128/null_importance}).

\end{document}